\newcommand*{\CVPR}{}

\newcommand*{\CAMREADY}{}

\ifdefined\NIPS
	\documentclass{article} 
	\usepackage{nips14submit_e,times}
	\usepackage{hyperref}
	\usepackage{url}
\fi
\ifdefined\CVPR
	\documentclass[10pt,twocolumn,letterpaper]{article}
	\usepackage{cvpr}
	\usepackage{times}
	\usepackage{epsfig}
	\usepackage{graphicx}
	\usepackage{amsmath}
	\usepackage{amssymb}
\fi
\ifdefined\AISTATS
	\documentclass[twoside]{article}
	\ifdefined\CAMREADY
		\usepackage[accepted]{aistats2015}
	\else
		\usepackage{aistats2015}
	\fi
	\usepackage{url}
\fi

\usepackage{color}
\usepackage{amsfonts}
\usepackage{amsmath}
\usepackage{algorithm}
\usepackage{algorithmic}
\usepackage{graphicx}
\usepackage{bbm}
\usepackage{amsthm}
\ifdefined\NIPS
	\usepackage[square,comma,numbers]{natbib}
\fi
\ifdefined\CVPR
	\usepackage[square,comma,numbers]{natbib}
\fi
\ifdefined\AISTATS
	\usepackage[round,authoryear]{natbib}
\fi

\newtheorem{definition}{Definition}
\newtheorem{lemma}{Lemma}

\newtheorem{theorem}{Theorem}

\def\be   {\begin{equation}}
\def\ee   {\end{equation}}
\def\beas {\begin{eqnarray*}}
\def\eeas {\end{eqnarray*}}
\def\bea {\begin{eqnarray}}
\def\eea {\end{eqnarray}}
\newcommand{\h} {{\mathbf h}}
\newcommand{\x} {{\mathbf x}}

\newcommand{\z} {{\mathbf z}}
\newcommand{\uu} {{\mathbf u}}
\newcommand{\vv} {{\mathbf v}}
\newcommand{\w}{{\mathbf w}}

\newcommand{\HH} {{\mathcal H}}
\newcommand{\X} {{\mathcal X}}

\newcommand{\C} {{\mathbb C}}
\newcommand{\R} {{\mathbb R}}
\newcommand{\N} {{\mathbb N}}
\newcommand{\1} {{\mathbf 1}}
\newcommand{\0} {{\mathbf 0}}
\newcommand{\K} {{\mathbf K}}
\newcommand{\Z} {{\mathbf Z}}
\newcommand{\alphabf}     {{\boldsymbol{\alpha}}}

\newcommand{\abs}[1]      {\lvert#1 \rvert}
\newcommand{\norm}[1]   {\left\|#1 \right\|}

\newcommand{\inprod}[2]  {\left\langle{#1},{#2}\right\rangle}
\DeclareMathOperator*{\argmax}{argmax}

\newcommand{\mexu}[2]  {\underset{#2}{MEX_{#1}}}

\ifdefined\CVPR
	\usepackage[breaklinks=true,bookmarks=false]{hyperref}
	\ifdefined\CAMREADY
		\cvprfinalcopy 
	\fi
	\def\cvprPaperID{****} 
	
\fi
\ifdefined\NIPS

	\ifdefined\CAMREADY
		\nipsfinalcopy
	\fi
\fi

\begin{document}

\ifdefined\NIPS
	\title{SimNets: A Generalization of Convolutional Networks}
	\author{
	Nadav Cohen \\
	The Hebrew University of Jerusalem \\
	\texttt{cohennadav@cs.huji.ac.il} \\
	\And 
	Amnon Shashua \\
	The Hebrew University of Jerusalem \\
	\texttt{shashua@cs.huji.ac.il} \\
	}
	\maketitle
\fi
\ifdefined\CVPR
	\title{SimNets: A Generalization of Convolutional Networks}
	\author{
	Nadav Cohen \\
	The Hebrew University of Jerusalem \\
	\texttt{cohennadav@cs.huji.ac.il} \\
	\and 
	Amnon Shashua \\
	The Hebrew University of Jerusalem \\
	\texttt{shashua@cs.huji.ac.il} \\
	}
	\maketitle
\fi
\ifdefined\AISTATS
	\twocolumn[
	\aistatstitle{SimNets: A Generalization of Convolutional Networks}
	\ifdefined\CAMREADY
		\aistatsauthor{Nadav Cohen \And Amnon Shashua}
		\aistatsaddress{The Hebrew University of Jerusalem \And The Hebrew University of Jerusalem}
	\else
		\aistatsauthor{Anonymous Author 1 \And Anonymous Author 2}
		\aistatsaddress{Unknown Institution 1 \And Unknown Institution 2}
	\fi
	]	
\fi

\begin{abstract}
We present a deep layered architecture that generalizes classical convolutional neural networks (ConvNets).  The architecture, called SimNets, is driven by two operators, one being a similarity function whose family contains the convolution operator used in ConvNets, and the other is a new soft max-min-mean operator called MEX that realizes classical operators like ReLU and max pooling, but has additional capabilities that make SimNets a powerful generalization of ConvNets.  Three interesting properties emerge from the architecture: (i) the basic input to hidden layer to output machinery contains as special cases kernel machines with the Exponential and Generalized Gaussian kernels, the output units being "neurons in feature space" (ii) in its general form, the basic machinery has a higher abstraction level than kernel machines, and (iii) initializing networks using unsupervised learning is natural.  Experiments demonstrate the capability of achieving state of the art accuracy with networks that are an order of magnitude smaller than comparable ConvNets.
\end{abstract}

\section{Introduction}
Convolutional neural networks (ConvNets) are attracting much attention largely due to their impressive empirical performance on large scale visual recognition tasks (c.f.~\cite{krizhevsky2012imagenet, zeiler2014visualizing, sermanet2013overfeat, taigman2013deepface, szegedy2014going}).  The ConvNet architecture has the capacity to model large learning problems that include thousands of categories, while employing prior knowledge embedded into the architecture.  The ConvNet capacity is controlled by varying the number of layers (depth), the size of each layer (breadth), and the size of the convolutional windows (which in turn are based on assumptions on local image statistics).  The learning capacity is controlled using over-specified networks (networks that are larger than necessary in order to model the problem), followed by various forms of regularization techniques such as Dropout~(\cite{hinton2012improving}).  

Despite their success in recent years, ConvNets still fall short of reaching the holy grail of human-level visual recognition performance.  Scaling up to such performance levels could take more than merely dialing up network sizes while relying on prior knowledge to compensate for what we cannot learn.  It may be worthwhile to challenge the basic ConvNet architecture, in order to obtain more compact networks for the same level of accuracy, or in other words, in order to increase the abstraction level of the basic network operations.  

A few observations have motivated our work.  The first is that the ConvNet architecture has not changed much since its early introduction in the 1980s~(\cite{lecun1995convolutional}) -- there were some attempts to create other types of deep-layered architectures (cf.~\cite{riesenhuber1999hierarchical, bruna2013invariant, poon2011sum}), but these are not commonly used compared to ConvNets.  Arguably, the empirical success that ConvNets have witnessed in recent years is mainly fueled by the ever-growing scale of available computing power and training data, with the contribution of algorithmic advances having secondary importance.  Our second observation is that although there were attempts to use unsupervised learning to initialize ConvNets (c.f.~\cite{hinton2006fast, bengio2007greedy, vincent2008extracting}), it has since been observed that these schemes have little to no advantage over carefully selected random initializations that do not use data at all (see for example~\cite{krizhevsky2012imagenet,zeiler2014visualizing,sermanet2013overfeat}).  We nevertheless believe that unsupervised initialization has an important role in scaling up the capacity of deep learning, and therefore find interest in deep architectures that give rise to natural initializations using unsupervised data.  The third observation that motivated our work is that the ConvNet learning paradigm completely took over classification engines developed in the 1990s like Support Vector Machines (SVM) and kernel machines in general.  These machine learning methods were well suited for ``flat'' architectures, and while attempts to apply them to deep layered architectures have been made~(\cite{cho2009kernel}), they did not keep up with the performance levels of the layered ConvNet architecture.  It may be beneficial to develop a deep architecture that includes the body of work on kernel machines, but which still has the capacity to model large learning problems like the ConvNet architecture.

In this paper we introduce a new family of layered networks we call SimNets (similarity networks).  The general idea is to ``lift'' the classical ConvNet architecture into something more general, a multilayer kernel network architecture, which carries several attractive features.  First, the architecture bridges the decades-old ConvNets with the statistical learning machinery of the last decade or so.  Second, it provides a higher level of abstraction than the convolutional and pooling layers of ConvNets, thus potentially providing more compact networks for the same level of accuracy.  Third, the architecture is endowed with a natural initialization based on unlabeled data, which also has the potential for determining the number of channels in each layer based on variance analysis of patterns generated from the previous layer.  In other words, the structure of a SimNet can potentially be determined automatically from (unlabeled) training data.

The SimNet architecture is based on two operators.  The first is analogous to, and generalizes, the convolutional operator in ConvNets.  The second, as special cases, plays the role of ReLU activation~(\cite{nair2010rectified}) and max pooling in ConvNets, but in addition, has capabilities that make SimNets much more than ConvNets.  In a set of limited experiments on CIFAR-10 dataset~(\cite{krizhevsky2009learning}) using a small number of layers, we achieved better or comparable performance to state of the art ConvNets with the same number of layers, and the specialized network studied in~\cite{coates2011analysis}, using 1/9 and 1/5, respectively, of the number of learned parameters.  

In the following sections, we introduce the two operators that the SimNet architecture comprises, and describe its special cases and properties.  The experiments section is still preliminary but demonstrates the power of SimNets and their potential for high capacity learning.  Additional experiments with deeper SimNets are underway, but those require extensive optimization and coding infrastructure in order to apply to large scale settings.

\section{The SimNet architecture } \label{sec:simnet_arch}
The SimNet architecture consists of two operators -- a ``similarity" operator that generalizes the inner-product operator found in ConvNets, and a soft max-average-min operator called MEX that replaces the ConvNet ReLU activation~(\cite{nair2010rectified}) and max/average pooling layers, and allows additional capabilities as will be described below.

The similarity operator matches an input $\x\in\R^d$ with a template $\z\in\R^d$ and a weight vector $\uu\in\R_+^d$ ($\R_+^d$ stands for the non-negative orthant of $\R^d$) through $\uu^\top\phi(\x,\z)$, where $\phi:\R^d\times\R^d\to\R^d$ is a similarity mapping.  We will consider two forms of similarity mappings: the ``linear'' form $\phi(\x,\z)_i=x_i z_i$, such that when setting $\uu=\1$ we obtain an inner-product operator, and the ``$l_p$'' form $\phi(\x,\z)_i=-\abs{x_i-z_i}^p$ defined for $p>0$.  

In a layered architecture, a similarity layer is illustrated in fig.~\ref{fig:layers_nets}(a), where the similarity operator is applied to patches $\x_{ij}\in\R^{hwD}$ of width $w$, height $h$ and depth $D$, with the indexes $(i,j)$ describing the location of the patch within the layer's input.  Given $n$ templates $\z_1,...,\z_n\in\R^{hwD}$ and weights $\uu_1,...,\uu_n\in\R_+^{hwD}$, the layer's output at coordinates $(i,j,l)$ becomes $\uu_l^\top\phi_l(\x_{ij},\z_l)$, where we use index $l$ in $\phi_l$ to indicate that the similarity mapping may differ across channels.  As customary with ConvNets, the width and height of the layer's output depends on the ``stride'' setting, which determines the step-size between input patches, e.g. with horizontal and vertical strides of $s$, the spatial dimensions of the output become $\lfloor(H-h)/s\rfloor+1$ and $\lfloor(W-w)/s\rfloor+1$.  Note that using the linear-similarity mapping with unit weights($\uu_l=\1$) reduces the similarity layer to a standard convolutional layer where $\z_l$ are the convolution kernels, whereas for $l_p$-similarity with fixed $p=2$, the output at coordinates $(i,j,l)$ measures the weighted Euclidean (Mahalanobis) distance between the input patch $\x_{ij}$ and the template $\z_l$ with every pixel weighted through the entries of the weight vector $\uu_l$.  When using $l_p$-similarity in general, fixing the order $p$ is not obligatory -- the order can be learned based on training data, either globally or independently for each output channel.  

We will see later on that, when setting unit weights, the (unweighted) linear and $l_p$ similarity mappings correlate with kernel-SVM methods of statistical machine learning (through special cases of the SimNet architecture), and that the view of $\z_l$ as templates allows natural unsupervised initialization of networks using conventional statistical estimation methods.

The MEX operator, whose name stands for Maximum-minimum-Expectation Collapsing Smooth (with ``CS'' pronounced as ``X''), is responsible for the role of activation functions, max or average pooling (both spatially and across channels), and weights necessary for classification.  The operator is defined as follows:
\be
\mexu{\xi}{i=1,...,n}\{c_i\}:=\frac{1}{\xi}\log\left(\frac{1}{n}\sum_{i=1}^n\exp\{\xi{\cdot}c_i\}\right)
\label{eqn:mex_def}
\ee
with the alternative notation $MEX_\xi\{c_i\}_{i=1}^n$ used interchangeably.  The parameter $\xi\in\R$ spans a continuum between maximum, expectation (mean), and minimum:
\beas
MEX_\xi\{c_i\}_{i=1}^n&\underset{\xi\to+\infty}{\longrightarrow}&\text{max}\{c_i\}_{i=1}^n  \\
MEX_\xi\{c_i\}_{i=1}^n&\underset{\xi\to0}{\longrightarrow}       &\text{mean}\{c_i\}_{i=1}^n \\
MEX_\xi\{c_i\}_{i=1}^n&\underset{\xi\to-\infty}{\longrightarrow} &\text{min}\{c_i\}_{i=1}^n 
\eeas
Moreover, for a given value of $\xi$, the operator is smooth and exhibits the ``collapsing'' property defined below:
\bea 
&MEX_\xi\{ MEX_\xi\{c_{ij}\}_{1\leq j\leq m}\}_{1\leq i\leq n}&  \nonumber \\
&=MEX_\xi\{c_{ij}\}_{1\leq j\leq m,1\leq i\leq n}& \label{eqn:mex_collapse}
\eea

In a layered architecture, the MEX operator is used to define the MEX layer -- see illustration in fig.~\ref{fig:layers_nets}(b). In the MEX layer, the input is divided into (possibly overlapping) blocks, each mapped to a single output element.  The output value associated with the $t$'th input block is given by:
$$out(t)=MEX_{\xi_t}\left\{\left\{inp(s)+b_{ts}\right\}_{s\in block(t)}, c_t\right\}$$
where the index $s$ runs though the input block, the offsets $b_{ts}\in\R$ serve various roles as will be described later, and $c_t\in\R$ are optional (may or may not be used).  The MEX layer can realize two standard ConvNet layers -- the ReLU activation and the max-pooling layer.  To realize ReLU activation, one should set the input blocks to be single entries, have the output dimensions equal to the input dimensions, set $b_{ts}=0,c_t=0$, and let ${\xi_t}\rightarrow+\infty$, and as a result $out(t)=\max\{inp(t),0\}$ as required (see fig.~\ref{fig:layers_nets}(c)).  To realize a max-pooling layer, set the input blocks to cover a 2D area, set the depth of the output equal to that of the input, set $b_{ts}=0$, omit $c_t$, and set ${\xi_t}\rightarrow+\infty$.  As a result $out(i,j,l)=\max\{inp(i',j',l)\}_{(i',j')\in pool(i,j)}$ (see fig.~\ref{fig:layers_nets}(d)).  Note that by setting ${\xi_t}\rightarrow 0$ one obtains an average-pooling layer, and moreover, the parameters $\xi_t$ can be learned (optimized) as part of the training process, allowing additional flexibility.

To recap, the layers corresponding to the two operators of the SimNet architecture -- similarity and MEX, can realize conventional ConvNets as follows:
\begin{itemize}
\item \emph{Convolutional layer}: use similarity layer with linear form $\phi_l(\x,\z)_i=x_i z_i$ and unit weights $\uu_l=\1$.
\item \emph{ReLU activation}: use MEX layer with $b_{ts}=0,c_t=0,\xi_t\rightarrow+\infty$ and single-entry input blocks. 
\item \emph{Max pooling layer}: use MEX layer with $b_{ts}=0,\xi_t\rightarrow+\infty$, $c_t$ omitted, and 2D input blocks.
\item \emph{Dense layer}: use similarity layer with the entire input as the only patch, linear form $\phi_l(\x,\z)_i=x_i z_i$ and unit weights $\uu_l=\1$.
\end{itemize}

Next, we make wider use of the two SimNet layers, taking us beyond classical ConvNets, exploring connections to classical statistical learning models with kernel machines.

\begin{figure*}
\includegraphics[width=\textwidth,height=\textheight,keepaspectratio]{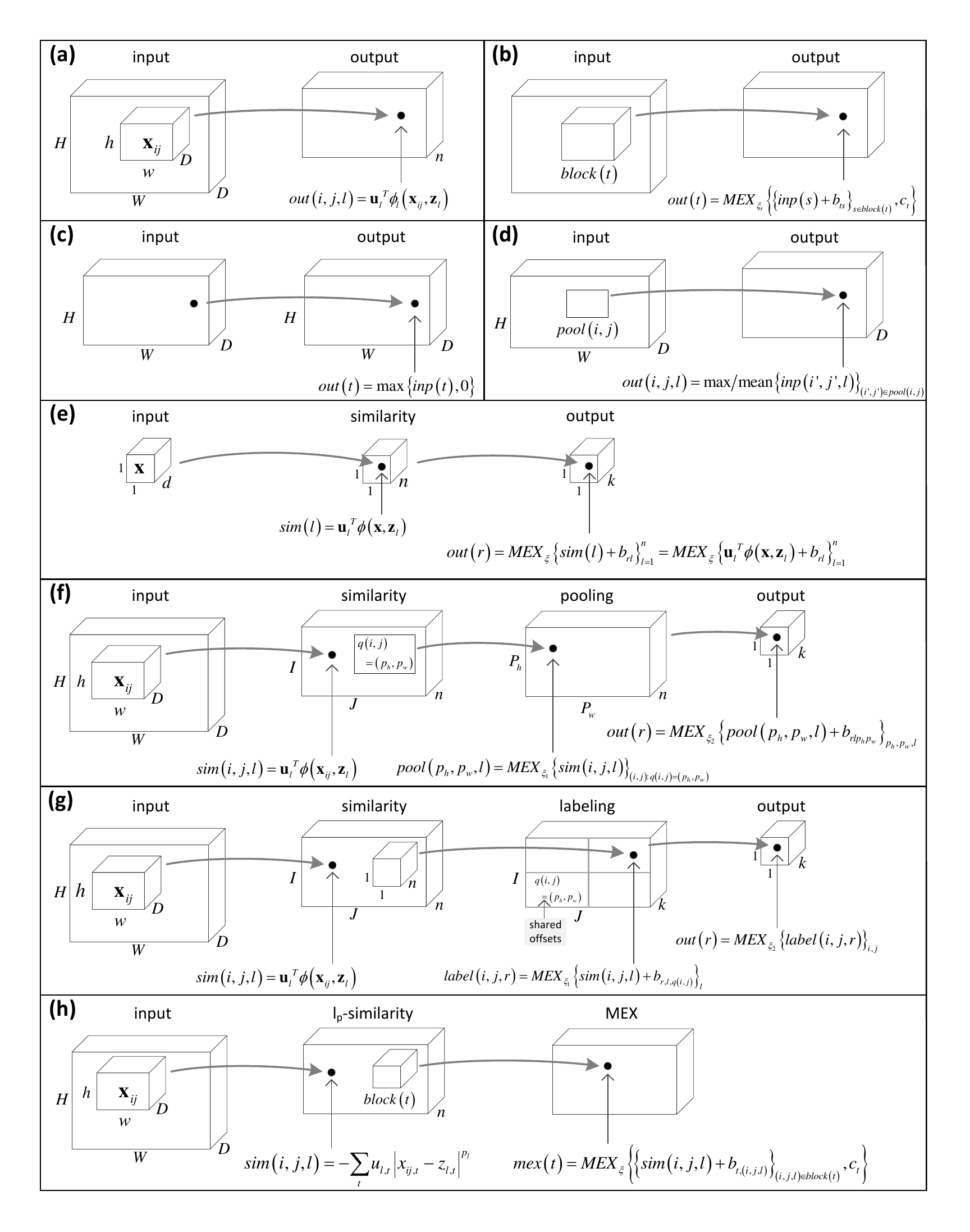}
\vspace{-1cm}
\caption{\footnotesize (a)~SimNet similarity layer~~(b)~SimNet MEX layer~~(c)~ConvNet ReLU activation layer~~(d)~ConvNet max/average pooling layer~~(e)~SimNet MLP with multiple outputs~~(f)~SimNet with locality, sharing and pooling~~(g)~Patch-labeling SimNet~~(h)~SimNet $l_p$-similarity layer followed by MEX layer}
\label{fig:layers_nets}
\end{figure*}

\section{SimNets and kernel machines} \label{sec:simnets_ksvm}
So far, we set the architectural choices of SimNets to realize classical ConvNets, which form a rudimentary special case of the available possibilities.  In particular, we did not make use of the $l_p$ similarity, and of the offsets $\{b_{ts}\}$ in the MEX layer.  In the following subsection, we consider a ``multi-layer perceptron'' (MLP) construction consisting of a single hidden layer in addition to input and output layers.  In the subsection that follows,  we will study the case where the input layer is processed by patches, the hidden layer involves locality and sharing, and a pooling operation follows the hidden layer -- a structure prevalent in classical ConvNets.
\subsection{MLP analogy: input $\rightarrow$ hidden layer $\rightarrow$ output} \label{subsec:mlp}
The Similarity and MEX operators give straightforward generalizations of the convolution and max/average pooling layers in ConvNets.  As we now show, they create something of greater consequence when applied one after the other in succession.  To make the point as succinctly as possible, consider a MLP construction consisting of $d$ input nodes (making up the input vector $\x\in\R^d$), $n$ hidden units, and a single output unit.  The value $h(\x)$ of the output unit is a result of a mapping $\R^d\to\R$, defined by the two SimNet operators applied in succession ($n$ similarity operators with different templates and shared mapping $\phi$, followed by MEX with offsets):
\be
h(\x)=MEX_\xi\left\{\uu_l^\top\phi(\x,\z_l)+b_l\right\}_{l=1,...,n} \label{eqn:hx}
\ee
A straightforward analogy to existing work is obtained by setting unit weights $\uu=1$, linear similarity $\phi(\x,\z)_i=x_i z_i$ and $\xi\to\infty$, resulting in $h(\x)=\max\left\{\z_l^\top\x+b_l\right\}_{l=1}^n$ -- a maxout operation~\cite{goodfellow2013maxout}.  There were other attempts to generalize maxout, notably the recently proposed $L_p$ unit~\cite{gulcehre2014learned}, which is defined by $\left(\frac{1}{n}\sum_{l=1}^n\abs{\z_l^\top\x+b_l}^p\right)^{1/p}$.  When $p\to\infty$, this reduces to $\max_l\left\{\abs{\z_l^\top\x+b_l}\right\}$.  The differences between this and the SimNet generalization of maxout are: (i) the $L_p$ unit generalizes maximum of \emph{absolute} values (rather than the values themselves), and (ii) the $L_p$ unit tries to create a maxout in a single operation whereas the SimNet creates $h(\x)$ over a succession of two operators -- similarity followed by MEX.

Next, consider the case of fixed $\xi>0$ and unweighted ($\uu_l=\1$) linear similarity ($\uu_l^\top\phi(\x,\z_l)=\x^\top\z_l$) or unweighted $l_p$ similarity ($\uu^\top\phi(\x,\z)=-\norm{\x-\z}_p^p$) with fixed $0<p\leq2$.  We will show below that in this case, the output $h(\x)$ is the result of a non-linear monotone activation function applied to the inner-product between a mapping of the input $\x$ and a vector $\w$ in some high-dimensional feature space $\R^F$.  More formally, we will show that $h(\x)=\sigma(\inprod{\w}{\psi_\phi(\x)})$, where the mapping $\psi_\phi:\R^d\to\R^F$ depends on the choice of similarity mapping $\phi$, $\sigma$ is a non-linear monotone activation function, and $\w=\sum_{l=1}^n\alpha_l\psi_\phi(\z_l)$ for some $\alpha_1,...,\alpha_n\in\R$.  We thus conclude that the output unit is a ``neuron'' in the classical sense, but in a high-dimensional feature space.  To prove this assertion, we notice that $h(\x)$ can be expressed as follows:
\bea
h(\x)
&=&MEX_\xi\left\{\uu_l^\top\phi(\x,\z_l)+b_l\right\}_{l=1,...,n}\nonumber \\
&=&\frac{1}{\xi}\ln\left(\frac{1}{n}\sum_{l=1}^n\alpha_l\cdot\exp\left\{\xi\sum_{i=1}^d\phi(\x,\z_l)_i\right\}\right) \nonumber\\
&=&\sigma\left(\sum_{l=1}^n\alpha_{l}\cdot K_\phi (\x,\z_l)\right) \label{eqn:hx_2} 
\eea
where $\alpha_l:=e^{\xi\cdot b_l}$ and $\sigma$ is a non-linear monotone activation function given by $\sigma(t)=\frac{1}{\xi}\ln\left(\frac{t}{n}\right)$.  We use the notation $K_\phi(\x,\z):=\exp\left\{\xi\sum_{i=1}^d\phi(\x,\z)_i\right\}$ to indicate that, under the similarity mappings considered, the function is a kernel on $\R^d$.  In particular, for the linear and $l_p$ similarities we have:
\beas
K_{lin}(\x,\z)  &=& \exp\left\{\xi\cdot\x^\top\z \right\}  \\
K_{l_p}(\x,\z) &=& \exp\left\{-\xi\norm{\x-\z_l}_p^p\right\}
\eeas
As shown in~\cite{scholkopf2002learning}, $K_{lin}$ and $K_{l_p}$ are kernels on $\R^d$ (note that for $p>2$, the expression above for $K_{l_p}$ is not a kernel).  We refer to them as the ``Exponential'' kernel and the ``Generalized Gaussian'' kernel\footnote{When $p=2$, this reduces to the well-known Gaussian (radial basis function) kernel.  When $p=1$, it reduces to the Laplacian kernel.} respectively.  Since $\sigma$ is monotonically increasing, $h(\x)$ realizes a 2-class ``reduced'' kernel-SVM decision rule, with $\z_l$ being the (reduced) support-vectors\footnote{We use the term ``reduced'' to refer to the case where the number of support-vectors is predetermined and they are not constrained to lie in the training set.  This setting was studied in~\cite{wu2006direct} in the context of binary (2-class) classification.  The extension to multiclass~(\cite{crammer2002algorithmic}) is straightforward.} and $\alpha_l\geq0$ being the coefficients associated with the support-vectors.  Let $\psi_\phi:\R^d\to\R^F$ be a feature mapping corresponding to $K_\phi$, i.e. $K_\phi(\x,\z)=\inprod{\psi_\phi(\x)}{\psi_\phi(\z)}$.  Eqn.~\ref{eqn:hx_2} can now be expressed as $h(\x)=\sigma(\inprod{\w}{\psi_\phi(\x)})$, where $\w:=\sum_{l=1}^n \alpha_l\psi_\phi(\z_l)$.  This shows that the output unit $h(\x)$ is a ``neuron is feature space'', as stated above.

In the case of weighted ($\uu_l$ are learned) $l_p$-similarity, the hypothesis space realized by the output unit $h(\x)$ is no longer representable by a kernel-SVM.  Moreover, the view of $h(\x)$ as a neuron in feature space with learned vector $\w$ no longer applies.  This is stated formally below (proof in app.~\ref{app:proof_l_p_sim_wgt_no_K}):
\begin{theorem}
\label{thm:l_p_sim_wgt_no_K}
For any dimension $d\in\N$, and constants $c>0$ and $p>0$, there are no mappings $Z:\R^d\to\R^d$ and $U:\R^d\to\R_+^d$ and a kernel $K:(\R^d\times\R_+^d)\times(\R^d\times\R_+^d)\to\R^d\times\R_+^d$, such that for all $\z,\x\in\R^d$ and $\uu\in\R_+^d$:
\be
K\left([Z(\x),U(\x)],[\z,\uu]\right)=\exp\left\{-c\sum_{i=1}^d u_i\abs{x_i-z_i}^p\right\} \label{eqn:l_p_sim_wgt_K_req}
\ee
\end{theorem}

We now turn to consider a straightforward extension to the setup above, which includes $k$ output units.  The MLP will now consist of an input signal $\x\in\R^d$, a set of $n$ hidden units defined by similarity functions over $\x$ (all based on the same similarity mapping $\phi$), and a set of $k$ output units defined by MEX operators (all having the same parameter $\xi$) with offsets $b_{rl}$ where $l\in\{1,...,n\}$ runs over the hidden units and $r\in\{1,...,k\}$ is the index of the output unit.  Fig.~\ref{fig:layers_nets}(e) illustrates this basic operation.  If we consider the output nodes as predicting a label associated with the input $\x$, the chosen label being the index of the node with maximal activation, then running the two operators, similarity and MEX, one following the other, produces the classification rule below:
\be
\hat{y}(\x)=\argmax_{r=1,...,k}MEX_{\xi}\{\uu_l^\top\phi(\x,\z_l)+b_{rl}\}_{l=1}^n 
\label{eqn:mlp_multiclass}
\ee
This classification measures weighted similarities to $n$ templates, with class-dependent offsets.  The role of the MEX operators is to combine the weighted similarities (with offsets) of the input $\x$ to the templates.  For example, when $\xi\to+\infty$, the classification rule is attracted to the most similar template where offsets assign relevancy of templates to classes.  Let $h_r(\x)$, $r=1,...,k$, be the value of output unit $r$ when the MLP is fed with input $\x$.  Following the lines of the derivation carried out for the single-output MLP, when working with unweighted linear similarity or unweighted $l_p$ similarity with $0<p\leq2$, it holds that:
$$ h_r(\x)=\sigma\left(\inprod{\w_r}{\psi_\phi(\x)}\right)$$
where $\w_r:=\sum_{l=1}^n\alpha_{rl}\psi_\phi(\z_l)$ and $\alpha_{rl}:=e^{\xi\cdot b_{rl}}$.  Moreover, the decision rule in eqn.~\ref{eqn:mlp_multiclass} can be expressed as:
\beas
\hat{y}(\x)
&=& \argmax_{r\in\{1,...,k\}}h_r(\x) \\
&=& \argmax_{r\in\{1,...,k\}}\inprod{\w_r}{\psi_\phi(\x)} \\
&=& \argmax_{r\in\{1,...,k\}}\sum_{l=1}^n\alpha_{rl} K_\phi(\x,\z_l)
\eeas
where $K_\phi$ is a kernel (Exponential or Generalized Gaussian) on $\R^d$.  This classification realizes the hypothesis space of a reduced multiclass kernel-SVM\footnote{Note that the coefficients $\alpha_{rl}$ are positive, whereas in classical multiclass SVM they may be any real numbers.  This however does not limit generality, as we can always add a common offset to all coefficients after SVM training is complete.}.

To summarize so far, we have shown that with linear similarity, the ``MLP'' construction consisting of input $\to$ hidden layer $\to$ output, gives rise to the hypothesis space of a (reduced) SVM with the Exponential kernel.  Replacing the linear similarity with unweighted $l_p$ similarity having fixed order $p$, gives rise to a kernel-SVM if and only if $p\leq2$, in which case the underlying kernel is the Generalized Gaussian kernel (the special cases of Gaussian and Laplacian kernels are obtained for $p=2$ and $p=1$ respectively).  With these similarities that give rise to kernel machines, a unit generated by similarity operators followed by MEX with offsets is a ``neuron in feature space''.  Finally, with weighted ($\uu_l$ are learned) $l_p$ similarity, the framework is no longer representable by a kernel-SVM.

To obtain a sense of the network's abstraction level, i.e. its ability to capture concept (category) distributions in input space, consider the classification rule in eqn.~\ref{eqn:mlp_multiclass} in the case $\xi\to+\infty$:
$$\hat{y}(x)=\argmax_{r=1,...,k}\max_{l=1,...,n}\{\uu_l^\top\phi(\x,\z_l)+b_{rl}\}$$
For any $r\in\{1,...,k\}$, denote by $A_r$ the decision region in input space that corresponds to class $r$, i.e. $A_r:=\{\x\in\R^d:\hat{y}(\x)=r\}$.  To understand the shape of $A_r$, we make the following definitions:
\bea
&A_{r,l}^{r',l'}:=& \label{eqn:decision_region_basic_shape} \\
&\{\x\in\R^d:\uu_l^\top\phi(\x,\z_l)+b_{rl}\geq\uu_{l'}^\top\phi(\x,\z_{l'})+b_{r'l'}\}& \nonumber
\eea
$$ A_{r,l}:=\bigcap_{(r',l')\neq(r,l)} A_{r,l}^{r',l'} $$
where the class index $r'$ ranges over $\{1,...,k\}$, and the template indexes $l,l'$ range over $\{1,...,n\}$.  One can readily see that up to boundary conditions:
$$ A_{r} = \bigcup_{l\in\{1,...,n\}}A_{r,l} $$

Consider first the setting of linear similarity ($\phi(\x,\z)_i=x_i z_i$).  In this case $A_{r,l}^{r',l'}$ are half-spaces and $A_{r,l}$ are intersections of half-spaces (polytopes).  The decision region $A_r$ is thus a union of $n$ polytopes.  As we now show, this is the same type of decision regions as obtained with unweighted $l_2$ similarity ($\uu_l=\1$, $\phi(\x,\z)_i=-\abs{x_i-z_i}^2$).  Indeed, in this case the term $\norm{\x}_2^2$ in both sides of the inequality defining $A_{r,l}^{r',l'}$ cancels-out, and we obtain again a half-space.  This in turn implies that as before, $A_{r,l}$ are polytopes and $A_r$ is a union of polytopes.  We conclude that with the MLP structure of: input $\to$ hidden layer $\to$ output units, the setting that realizes a Gaussian kernel machine (unweighted $l_2$ similarity), is qualitatively equivalent to the ``ConvNet'' (linear similarity) setting that realizes an Exponential kernel machine.  The difference in kernels does not account for any material difference in the network's hypothesis space, i.e. its abstraction level.

Remaining with $l_2$ similarity, we now consider the weighted setting, i.e. the setting in which $\uu_l$ are not fixed.  Thm.~\ref{thm:l_p_sim_wgt_no_K} tells us that in this case the hypothesis space is no longer governed by kernel-SVM.  From the decision region point-of-view, it is not difficult to see that in this case $A_{r,l}^{r',l'}$ is no longer a half-space, but a region defined by a second-order hyper-surface.  This implies that the set $A_{r,l}$ is no longer a polytope, and in particular is not necessarily convex.  The possible shapes that the decision region $A_r$ can take are thus enriched.  We conclude that unlike in the case of unweighted $l_2$ similarity, the setting of weighted $l_2$ similarity is characterized by an abstraction level higher than that induced by linear similarity (convolutional operator).

In the general setting of $l_p$ similarity, the sets $A_{r,l}^{r',l'}$ are more complex, and may be governed by non-convex non-smooth separating hyper-surfaces.  The full analysis is outside the scope of this paper, but an informal illustration of how the space is divided for $p=1$ and $d=2$ ($\R^d$ is the 2D plane) is given in fig.~\ref{fig:l1_decision_regions}.  Under this specific setup, the 2D plane is divided into two by a piece-wise linear separating boundary.  The unweighted case (uniform weights) is shown in fig.~\ref{fig:l1_decision_regions}(a).  In this case the space is divided equally (up to a shift caused by the offsets $b_{rl},b_{r'l'}$) based on the $l_1$ (Manhattan) distance metric.  Adding weights deforms the boundary line, where the higher the weights associated with a template ($\z_l$ or $\z_{l'}$) are, the less space is allocated to that template.  For example, in fig.~\ref{fig:l1_decision_regions}(d) the weights associated with the template $z_{l'}$ are uniformly high, thereby creating a small aperture in the 2D plane around that template.  Given that $A_{r,l}^{r',l'}$ is highly non-convex in the weighted setting, we expect weighted $l_1$ similarity to provide a higher abstraction level than that of linear similarity (convolutional operator).

\begin{figure*}
\includegraphics[width=\textwidth,height=\textheight,keepaspectratio]{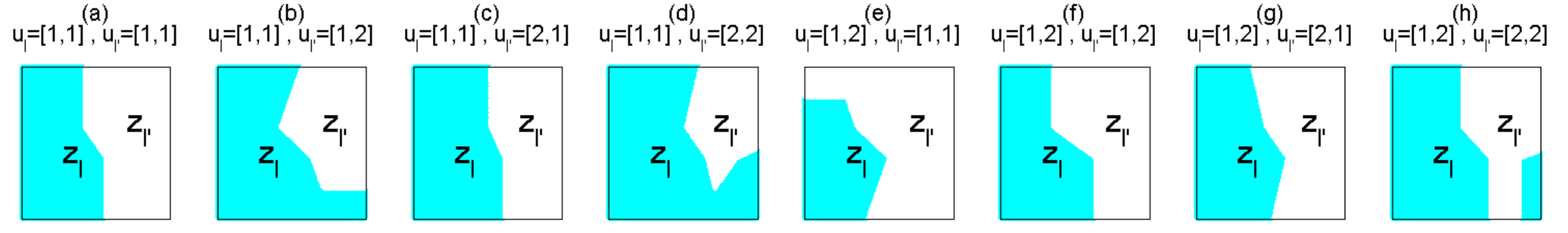}
\vspace{-3mm}
\caption{\footnotesize Illustration of $A_{r,l}^{r',l'}$ (defined in eqn.~\ref{eqn:decision_region_basic_shape}) in the setting of $l_1$ similarity and $d=2$ ($\R^d$ -- the 2D plane).  Each panel shows the location of the templates $\z_l$ and $\z_{l'}$, and on the top the values of the corresponding similarity weights.~~(a):~Unweighted setting (uniform weights).  The 2D plane is divided equally between the two templates (up to a shift resulting from the offsets $b_{r,l},b_{r',l'}$.~~(d):~Here $\z_{l'}$ is associated with high weights, thus the portion of the plane allocated to this template is ``shrinked''.}
\label{fig:l1_decision_regions}
\end{figure*}
\subsection{A basic 3-layer SimNet with locality, sharing and pooling} \label{subsec:local_share_pool}
Next, we analyze a 3-layer SimNet with locality, sharing and pooling.  The network's input is processed by patches (locality), with the same templates and weights applied to all patches (sharing), thereby creating a stack of feature maps (channels) -- one for each template.  Spatial regions of each feature map are then pooled together to reduce dimensionality, and finally, a classification output layer predicts the label of the input.  We will show that such a network, consisting of input $\to$ feature maps $\to$ pooling $\to$ output, also corresponds to a kernel-SVM, with the kernels designed for a ``patch-based'' representation of the input signal. 

Locality, sharing and pooling are realized in the conventional manner.  Namely, the input is divided into (possibly overlapping) patches $\x_{ij}\in\R^d$ where $d=h\cdot w\cdot D$, with $h,w$ being the height and width of the patches.  A similarity layer as illustrated in fig.~\ref{fig:layers_nets}(a), but with the same similarity mapping $\phi$ for all channels, matches the patch $\x_{ij}$ with the template $\z_l\in\R^d$ (which is now a template representing a local patch in the layer's input) using the weights $\uu_l\in\R_+^d$, and the resulting value $\uu_l^\top\phi(\x_{ij},\z_l)$ is stored in coordinates $(i,j,l)$ of the layer's output.

The mapping from the similarity layer to the $k$-node classification output is realized through two MEX layers.  The first MEX layer implements a pooling layer as follows.  Let $q(i,j)=(q_h(i),q_w(j))$ be a (contraction) mapping of the 2D coordinate system in the similarity layer to the 2D coordinate system in the pooling layer.  Normally, a 2D coordinate in the pooling layer corresponds to a 2D window in the similarity layer.  The value assigned to an element in the pooling layer is simply a MEX operation taken over the corresponding 2D window (in the respective channel $l\in\{1,...,n\}$) in the similarity layer:
$$pool(p_h,p_w, l)=MEX_{\xi_1}\{sim(i,j,l)\}_{i,j: q(i,j)=(p_h,p_w)}$$
All MEX operators in the layer have the same parameter -- $\xi_1$.  When $\xi_1\rightarrow+\infty$ for instance, we obtain max-pooling as implemented in conventional ConvNets. 

The second MEX layer implements a dense mapping from the pooling layer to the $k$ output nodes, which includes offsets.  The value of the $r$'th output node is given by:
$$out ( r )= MEX_{\xi_2}\{pool(p_h,p_w,l)+b_{r l p_h p_w}\}_{p_h,p_w,l}$$
where $(p_h,p_w)$ runs over the 2D coordinates of the pooling layer and $l$ runs over the pooling layer's channels (which correspond to templates).  Note that here too all MEX operators have the same parameter -- $\xi_2$.  The offsets $b_{r l p_h p_w}$ depend on the output node $r$ and on the 3D coordinates of the pooling layer $(p_h,p_w,l)$, i.e. for every output node there is an offset for each coordinate in the pooling layer.

The SimNet we obtain is illustrated in fig.~\ref{fig:layers_nets}(f).  It is a basic similarity-pooling-output network that employs locality and sharing, as conventional ConvNets do.  The point to be made next, is that in the special case where $\uu_l=\1$, $\xi_1,\xi_2$ are fixed to a constant $\xi>0$, and $\phi$ is set to linear form or $l_p$ form with fixed $p\leq2$, the classification resulting from this network is a kernel-SVM, where the kernel is designed for a ``patch-based'' representation of the input signal.  First, by concatenating the three steps -- similarity, pooling and output, and assuming that $\xi_1=\xi_2=\xi>0$, the decision rule associated with the network becomes:
\be
\hat{y}(inp)=\argmax_{r=1,...,k}\mexu{\xi}{i,j,l}\left\{\uu_l^\top\phi(\x_{ij},\z_l)+b_{r,l,q(i,j)}\right\} 
\label{eqn:simnet_patch_class}
\ee
This follows from the identities below:
\beas
&\mexu{\xi}{p_h,p_w,l}\left\{\mexu{\xi}{i,j: q(i,j)=(p_h,p_w)}\left\{\uu_l^\top\phi(\x_{ij},\z_l)\right\}+b_{r l p_h p_w}\right\}& \\
&=\mexu{\xi}{p_h,~p_w,~l,~i,j:q(i,j)=(p_h,p_w)}\left\{\uu_l^\top\phi(\x_{ij},\z_l)+b_{r l p_h p_w}\right\}& \\
&=\mexu{\xi}{i,j,l}\left\{\uu_l^\top\phi(\x_{ij},\z_l)+b_{r,l,q(i,j)}\right\}&
\eeas
where for the first equality, we used the collapsing property of the MEX operator described in eqn.~\ref{eqn:mex_collapse}.  The classification described in eqn.~\ref{eqn:simnet_patch_class} is similar to that described in eqn.~\ref{eqn:mlp_multiclass}, but has two important distinctions: \emph{(i)} the templates $\z_l$ are local and similarity is applied locally (hence the ``locality'' and ``sharing''), and \emph{(ii)} the offsets are region-based (hence the ``pooling''), i.e. each collection of input patches $\x_{ij}$ ascribed to the same pool is associated with a single set of offsets (per-class and per-template).  To see the kernel structure associated with this classification, we perform the following manipulations to the rule given in eqn.~\ref{eqn:simnet_patch_class}:
\bea
&\hat{y}(inp)=& \nonumber \\
&\argmax\limits_{r=1,...,k}\mexu{\xi}{i,j,l}\left\{\uu_l^\top\phi(\x_{ij},\z_l)+b_{r,l,q(i,j)}\right\}=& \nonumber \\
&\argmax\limits_{r=1,...,k}\sum\limits_{i,j,l}e^{\xi\cdot(\uu_l^\top\phi(\x_{ij},\z_l)+b_{r,l,q(i,j)})}=& \nonumber \\
&\argmax\limits_{r=1,...,k}\sum\limits_{i,j,l}\underbrace{e^{\xi\cdot b_{r,l,q(i,j)}}}_{:=\alpha_{r,l,q(i,j)}}\cdot e^{\xi\cdot\uu_l^\top\phi(\x_{ij},\z_l)}=& \nonumber \\
&\argmax\limits_{r=1,...,k}\sum\limits_{p_h,p_w,l}\alpha_{r l p_h p_w}\sum\limits_{i,j:q(i,j)=(p_h,p_w)} e^{\xi\cdot\uu_l^\top\phi(\x_{ij},\z_l)}& \label{eqn:simnet_patch_class_2}
\eea
Setting $\uu_l=\1$, and referring to subsec.~\ref{subsec:mlp}, we denote $K_\phi(\x_{ij},\z_l):=\exp\{\xi\cdot\1^\top\phi(\x_{ij},\z_l)\}$, emphasizing that this function is a kernel for the similarity mappings we consider (Exponential kernel for linear similarity, Generalized Gaussian kernel for $l_p$ similarity with fixed $p\leq2$).  Eqn.~\ref{eqn:simnet_patch_class_2} then becomes:
\bea 
&\hat{y}(inp)=& \label{eqn:simnet_patch_class_3} \\
&\argmax\limits_{r=1,...,k}\sum\limits_{p_h,p_w,l}\alpha_{r l p_h p_w}\sum\limits_{i,j:q(i,j)=(p_h,p_w)}K_\phi(\x_{ij},\z_l)& \nonumber \\
&=\argmax\limits_{r=1,...,k}\sum_{p_h,p_w,l}\alpha_{r l p_h p_w}\K_\phi(X,Z_{l p_h p_w})& \nonumber
\eea
where $X$ contains the concatenation of all the input patches $\x_{ij}$, and $\Z_{l p_h p_w}$ is a structure containing copies of $\z_l$ in locations corresponding to the pool index $(p_h,p_w)$ -- the details, including definition of $\K_\phi$ and proof that it is indeed a kernel, are given in app.~\ref{app:patch_ksvm}.

\section{Other SimNet settings -- global average pooling } \label{sec:other_simnets}
In subsec.~\ref{subsec:local_share_pool} we introduced the SimNet basic building chain of the form: input $\to$ similarity $\to$ pooling $\to$ output, whose structure follows the line of classical ConvNets.  We noted that the basic building chain realizes a kernel-SVM hypothesis space, where the templates in the similarity layer correspond to the (reduced) support-vectors, and the offsets in the last MEX layer (from pooling to output) are related to the SVM coefficients.  The SVM hypothesis space is realized when the similarity operator is set to linear form or $l_p$ form with fixed $p\leq2$, and is unweighted ($\uu_l=\1$).  Using weighted $l_p$ similarity (weights are not applicable to linear similarity) has the potential of providing a richer hypothesis space than kernel-SVM (at the expense of doubling the number of parameters in the similarity layer).  Indeed, experiments we conducted (reported in sec.~\ref{sec:exp}) validate the power of $l_p$ similarity weighting, showing that  it matters more than merely the added number of parameters to the model.

In this section, we introduce another SimNet building chain with two MEX layers, designed in such a way that when the MEX parameters are equal, the chain collapses into the one presented above (decision rule in eqn.~\ref{eqn:simnet_patch_class}), but when the MEX parameters are determined separately -- either learned using training data or set manually, the SimNet chain allows for new possibilities (without additional parameters).  For example, setting the MEX parameter of the first layer to $1$ and that of the second layer to $0$ gave rise to the best experimental performance we encountered.

The idea is to switch the roles of the two MEX layers -- rather than having the first play the role of pooling and the second the role of classification (using the offsets $b_{rl}$), we start with a MEX layer with offsets and finish with a MEX for pooling.  The interpretation of such a structure is that each input patch $\x_{ij}$ undergoes classification in the first MEX layer.  The second MEX layer performs a majority voting over all the patch-based classification results to form a final classification decision.  This approach follows the line of the ``global average pooling'' structure recently suggested in the context of ConvNets, which has been shown to outperform the traditional ``dense classification'' paradigm~(\cite{lin2013network,szegedy2014going}).  To enforce spatial consistency in the labeling characteristics of patches, we constrain the first MEX layer's offsets to be uniform inside predetermined spatial regions.  The resulting SimNet, which we refer to as a ``patch labeling'' network, is illustrated in fig.~\ref{fig:layers_nets}(g) (note that all channels in the similarity layer share the same similarity mapping, and that both MEX layers have global parameters $\xi_1,\xi_2$).  Its classification rule takes the following form:
$$ \hat{y}(inp)=\argmax_{r=1,...,k}out(r) $$
with:
\beas
&out(r)=& \\
&\mexu{\xi_2}{i,j}\{\mexu{\xi_1}{l}\{\uu_l^\top\phi(\x_{ij},\z_l)+b_{r,l,q(i,j)}\}\}& 
\eeas
The variables which can be learned here are the offsets $b_{r l p_h p_w}\in\R$ (with $r$ ranging over the classes, $l$ over the templates and ($p_h,p_w)$ over the regions in which offsets are shared), the templates $\z_l\in\R^{hwD}$, the similarity weights $\uu_l\in\R_+^{hwD}$, the order $p>0$ in case $l_p$ similarity is chosen, and the MEX parameters $\xi_1,\xi_2\in\R$.  Assume we constrain the MEX parameters to be equal: $\xi_1=\xi_2=\xi$.  The MEX collapsing property (eqn.~\ref{eqn:mex_collapse}) then applies, and the classification rule becomes:
$$ \hat{y}(inp)=\argmax_{r=1,...,k}MEX_{\xi}\{\uu_l^\top\phi(\x_{ij},\z_l)+b_{r,l,q(i,j)}\}_{i,j,l} $$
which is identical to the decision rule in eqn.~\ref{eqn:simnet_patch_class}.  However, there is no reason to have the MEX parameters equal to each other.  We can estimate their value during training, or set them manually.  For example, during our experimentation we found that the case of equal MEX parameters -- $\xi_1=\xi_2=\xi$, is significantly outperformed by the setting $\xi_2\to0$, which corresponds to the following classification rule:
\beas
&\hat{y}(inp)=& \\
&\argmax\limits_{r=1,...,k}\sum_{i,j}\mexu{\xi_1}{l\in\{1,...,n\}}\left\{\uu_l^\top\phi(\x_{ij},\z_l)+b_{r,l,q(i,j)}\right\}& 
\eeas

\section{Initializing parameters using unsupervised learning} \label{sec:init}
For classical ConvNets, various schemes of initializing a network based on unlabeled data (unsupervised initialization) have been proposed (c.f.~\cite{hinton2006fast,bengio2007greedy,vincent2008extracting}).  Over time, however, these were taken over by carefully selected random initializations that do not use data at all (see for example~\cite{krizhevsky2012imagenet,zeiler2014visualizing,sermanet2013overfeat}).  No initialization scheme to-date is sufficient on its own for overcoming the hardness of training.  Indeed, successful training of ConvNets typically requires designing an over-specified network (i.e. a network that is much larger than necessary in order to represent the true hypothesis space).  While the latter has been shown to produce good training results, it bares a computational price, and also aggravates the problem of overfitting.  The enhanced susceptibility to overfitting has led to various regularization techniques and heuristics (Dropout~(\cite{hinton2012improving}) being the most prominent), which nowadays form an art that one must master in order to properly train ConvNets.  In this section, we discuss a natural unsupervised initialization scheme for SimNets, which is based on statistical estimation.  Such a scheme may provide a more effective local minima in the process of training a SimNet, thereby reducing the need for over-specification, supporting smaller networks that are more efficient computationally, and less prone to overfit.  Experiments we conducted (reported in sec.~\ref{sec:exp}) validate this conjecture, showing that the SimNet unsupervised initialization scheme indeed improves performance over random initializations, especially in the case of small networks.

Recall from sec.~\ref{sec:simnet_arch} that measuring weighted similarities to templates forms the similarity layer -- a basic building block of the SimNet architecture (see fig.~\ref{fig:layers_nets}(a)).  Focusing on the case of $l_p$ similarity mappings ($\phi(\x,\z)_i = -\abs{x_i-z_i}^p$), we show how the application of statistical estimation methods to unlabeled training data can produce initialization values for the layer's templates $\z_1,...,\z_n$,  weights $\uu_1,...,\uu_n$ and orders $p_1,...,p_n$.  Consider a probability distribution over $\R^d$ defined by a mixture of $n$ Generalized Gaussian distributions, each having independent coordinates with a shared shape parameter and separate scales and means:
$$P(\x)=\sum_{l=1}^n\lambda_l\prod_{i=1}^d\frac{\beta_l}{2\alpha_{l,i}\Gamma(1/\beta_l)}
\exp\left\{-\left(\frac{\abs{x_i-\mu_{l,i}}}{\alpha_{l,i}}\right)^{\beta_l}\right\}$$
In the above, $\lambda_l$ stands for the prior probability of component $l$ ($\lambda_l\geq0,\sum_l\lambda_l=1$), $\beta_l>0$ stands for the shape parameter of all coordinates in component $l$, $\alpha_{l,i}>0$ stands for the scale of coordinate $i$ in component $l$, $\mu_{l,i}\in\R$ stands for the mean of coordinate $i$ in component $l$, and $\Gamma$ is the Gamma function, defined by $\Gamma(s)=\int_0^{\infty}e^{-t}t^{s-1}dt$.  The log-probability that a vector drawn from this distribution is equal to $\x$ and originated from component $l$ is: 
$$\log P(\x\land \text{component}~l)=-\sum_{i=1}^d\alpha_{l,i}^{-\beta_l}\abs{x_i-\mu_{l,i}}^{\beta_l}+c_l$$
where $c_l:=\log\left\{\lambda_l\prod_{i=1}^d\frac{\beta_l}{2\alpha_{l,i}\Gamma(1/\beta_l)}\right\}$ is a constant that depends on the component only (not on $\x$).  Setting the layer's templates by $z_{l.i}=\mu_{l,i}$, its weights by $u_{l,i}=\alpha_{l,i}^{-\beta_l}$ and its $l_p$ orders by $p_l=\beta_l$, would give:
$$ \uu_l^\top\phi_l(\x_{ij},\z_l)=\log P(\x\land \text{component}~l)-c_l $$
This implies that if we assume input patches follow a Generalized Gaussian mixture as described, initializing the similarity layer's templates, weights and orders as above would result in channel $l$ of the layer's output holding, up to a constant, the probabilistic heat map of component $l$ and the patches.  This observation suggests estimating the parameters (shapes $\beta_l$, scales $\alpha_{l,i}$ and means $\mu_{l,i}$) of the Generalized Gaussian mixture using unlabeled input patches (via standard statistical estimation methods, such as that presented in~\cite{bazi2007image}), and initializing the similarity layer accordingly.
  
Consider now the case where the initialized $l_p$ similarity layer is followed by a MEX layer with learned offsets (see fig.~\ref{fig:layers_nets}(h), where for convenience, the linear index $t$ is used to refer to elements of the MEX layer's 3D output array).  We now assume that not only do input patches come from a mixture of Generalized Gaussian components as above, but also that each input patch location corresponds to a different mixture (priors) of these components.  This makes sense, as certain templates that are likely to appear in the center of an image for example, may be less likely to appear on the top-left corner of the image, for example.  Using our estimates of the global components obtained during the initialization of the similarity layer, we can estimate a mixture for a certain input patch location, by applying an estimation method to patches only from that location, with the component shapes, means and scales held fixed.  We may then calculate offsets for the $n$ elements of the similarity layer's output that correspond to that location, such that the probabilistic heat maps will take into account the location-dependent statistics, and will be precise (not up to a constant).  For example, if there is a region in the input for which a certain template is very unlikely to appear, that template's heat map in the aforementioned region will be suppressed.  The offsets we compute may serve for initialization of the MEX layer's offsets.

\section{Experiments} \label{sec:exp}
\begin{figure*} 
\includegraphics[width=\textwidth,height=\textheight,keepaspectratio]{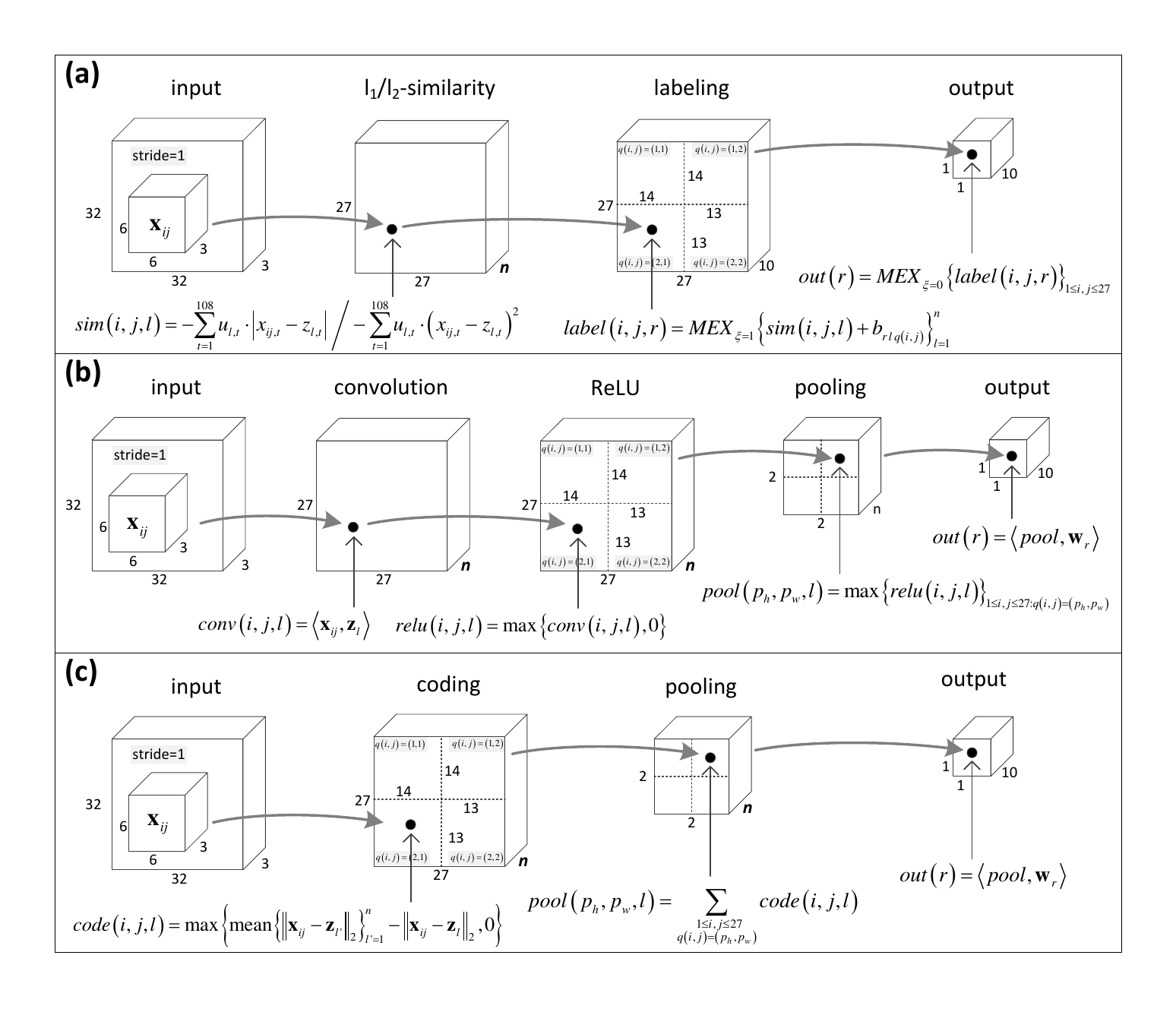}
\vspace{-1cm}
\caption{\footnotesize Networks evaluated on CIFAR-10.~~(a)~Patch labeling SimNet~~(b)~Comparable ConvNet~~(c)~Comparable ``single-layer'' network studied in~\cite{coates2011analysis}.} 
\label{fig:exp_nets}
\end{figure*}
\begin{figure*} 
\includegraphics[width=\textwidth,height=\textheight,keepaspectratio]{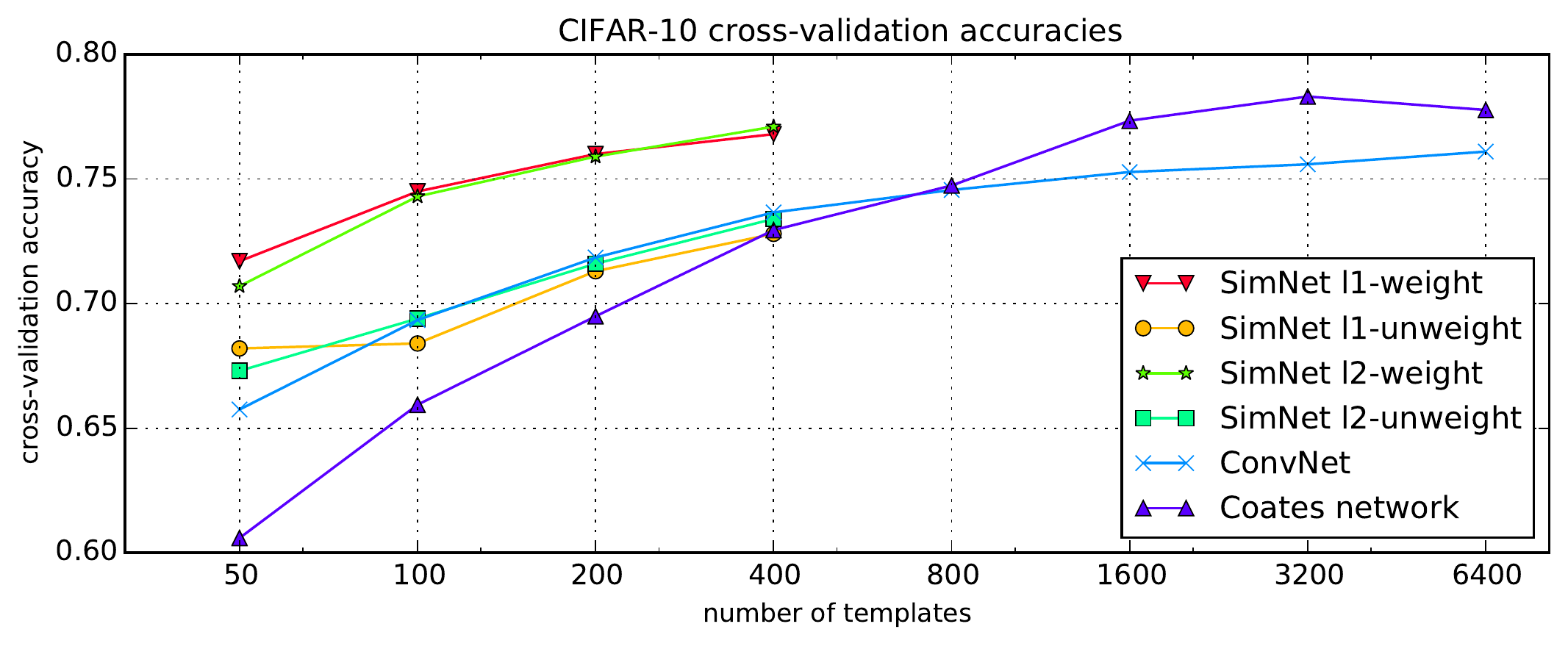}
\vspace{-5mm}
\caption{\footnotesize CIFAR-10 cross-validation accuracies plotted against the number of templates in the networks (denoted $n$ in fig.~\ref{fig:exp_nets}).  `SimNet l1-weight' and `SimNet l2-weight' correspond to the network structure illustrated in fig.~\ref{fig:exp_nets}(a), with $l_1$ and $l_2$ similarities respectively; 'SimNet l1-unweight' and 'SimNet l2-unweight' correspond to the same networks, but with the weight vectors $\uu_l$ held fixed during training; 'ConvNet' corresponds to the network illustrated in fig.~\ref{fig:exp_nets}(b); 'Coates network' corresponds to the network illustrated in fig.~\ref{fig:exp_nets}(c).}
\label{fig:benchmark_results_templates}
\end{figure*}
\begin{figure*} 
\includegraphics[width=\textwidth,height=\textheight,keepaspectratio]{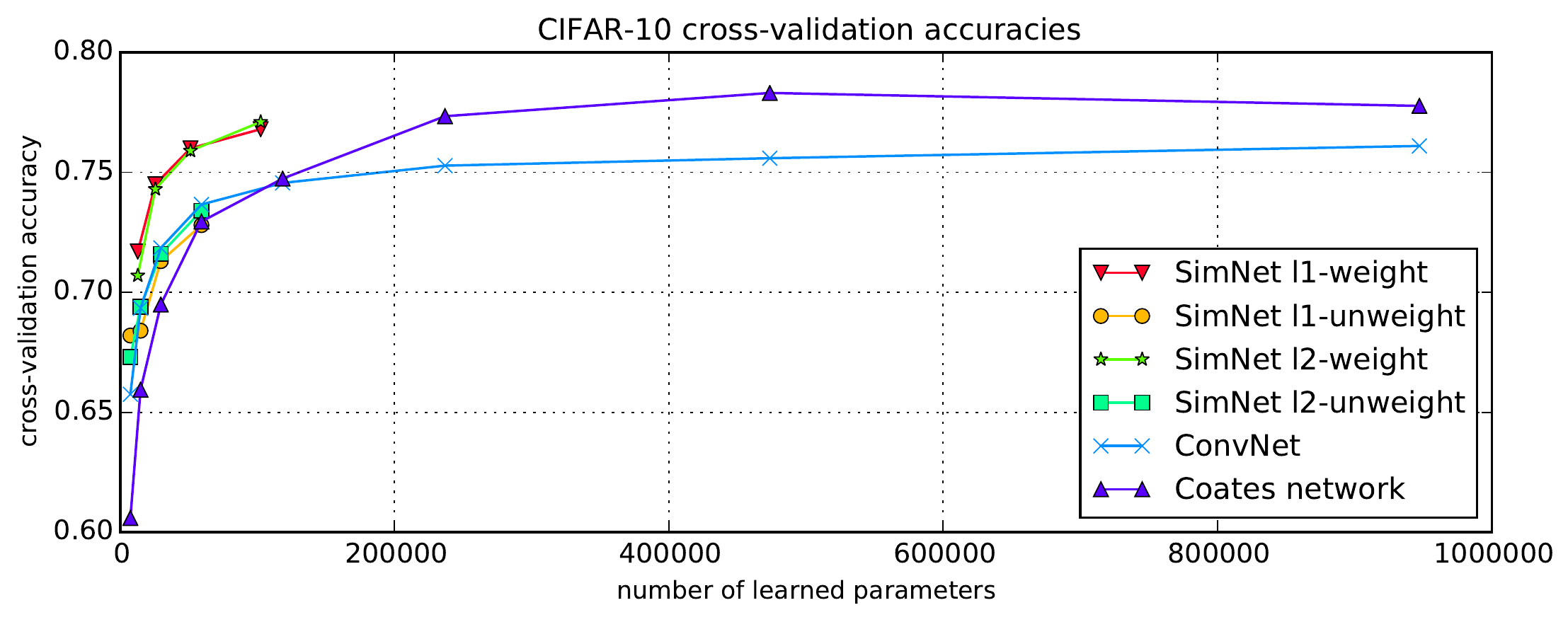}
\vspace{-5mm}
\caption{\footnotesize CIFAR-10 cross-validation accuracies plotted against the number of learned parameters in the networks.  This is merely a different display of the results given in fig.~\ref{fig:benchmark_results_templates}.  Notice how with weighted similarities, the SimNet reaches approximately the same level of performance as the competition, using much smaller networks.}
\label{fig:benchmark_results_params}
\end{figure*}

We implemented the ``patch labeling'' SimNet discussed in sec.~\ref{sec:other_simnets}, and experimented with the specific architectural settings illustrated in fig.~\ref{fig:exp_nets}(a).  The network consists of a $l_p$-similarity layer with $p$ fixed at $1$ or $2$, followed by two MEX layers.  Implementation and evaluation of deeper SimNets is currently under work, and will be reported at a later time.  For the experiments reported here, we used the CIFAR-10 dataset~(\cite{krizhevsky2009learning}), which consists of $60,000$ color images ($50,000$ for training and $10,000$ for testing) of size $32\times32$ partitioned into $10$ classes, with $6,000$ images per class ($5,000$ for training, $1,000$ for testing).  The network's input is an RGB image ($32\times32\times3$ array), processed by patches of size $6\times6\times3$ with a single-pixel stride between them.  For a given number of templates in the similarity layer (denoted by $n$), the SimNet's learned parameters are the templates $\z_1,...,\z_n\in\R^{108}$, the similarity weights $\uu_1,...,\uu_n\in\R_+^{108}$, and the MEX offsets $b_{r l p_h p_w}\in\R$ with $r=1,...,10$, $l=1,...,n$ and $p_w,p_h=1,2$.  We used statistical estimation as described in sec.~\ref{sec:init} to initialize the templates $\z_l$ and the similarity weights $\uu_l$ (initialization was based on training images, without making use of their labels).  The network was then trained by minimizing a softmax loss with stochastic gradient descent (SGD) that includes momentum and acceleration~(\cite{sutskever2013importance}).  For SGD, we used a batch size of $64$, a momentum of $0.9$, and a learning rate of $0.01$ decreased by a factor of $10$ after $50$ epochs, running $100$ epochs in total.  The weight decay for the templates was set to zero, and those for the similarity weights and offsets were set equal to each other, their value chosen via cross-validation.  

We compared the SimNet to instances of two learning architectures.  The first is a ConvNet with a single convolutional layer followed by a pooling layer followed by an output layer (see illustration in fig.~\ref{fig:exp_nets}(b)).  The purpose of this comparison is to evaluate the SimNet against an analogous ConvNet, measuring the network sizes (number of learned parameters) required to reach given accuracies.  A successful outcome here would be if the SimNet reached the same (or higher) level of performance as the ConvNet, with considerably smaller network size.  The second comparison we held was against the ``single-layer'' network studied by Coates et al.~(\cite{coates2011analysis}), which has the same depth as the evaluated SimNet, and whose performance on CIFAR-10 is one of the best reported for networks of such depth (absolute state of the art in 2011).  In~\cite{coates2011analysis}, a number of unsupervised learning methods were devised for ``coding'' the input image.  The coding methods were based on ``templates'', such that each template corresponded to a single feature map.  The feature maps were passed on to a sum-pooling operator, and from there a linear SVM was learned using supervised data.  Many coding methods were experimented on, and the one that produced the best results, referred to as ``triangle'' coding, was a ``soft'' Euclidean measure applied to templates learned via k-means.  This coding method, along with the other architectural settings that produced the best results, are illustrated in fig.~\ref{fig:exp_nets}(c).  Finally, we question the importance of the unsupervised initialization scheme described in sec.~\ref{sec:init}, by training the evaluated SimNet with random initialization (as customary with ConvNets), and examining the effect on the cross-validation accuracies.

The results reported below show that the evaluated SimNet achieves performance comparable to that of the ConvNet and the network of Coates et al., with only a fraction of the number of learned parameters.  The unsupervised initialization scheme indeed boosts performance, and can be viewed as one of the drivers behind the SimNet's superiority.  We are currently working on the optimization of our code (including GPU acceleration), to enable evaluation of larger and deeper SimNets on more meaningful benchmarks, comparing against deep state of the art ConvNets.
\subsection{Benchmarking against the ConvNet}
The ConvNet was implemented using Caffe toolbox~(\cite{Jia13caffe}), with random initialization and SGD training.  We used a batch size of $100$, momentum of $0.95$, and learning rate of $10^{-4}$ decreased by a factor of $10$ every $45$ epochs, running $150$ epochs in total.  The global weight decay and the dense layer's DropOut rate were chosen via cross-validation.  The ConvNet's input was an RGB image ($32\times32\times3$ array) normalized for brightness and contrast.  For the SimNet we also added patch ``whitening'', in accordance with the suggestion of~\cite{coates2011analysis}.  The positive effect of whitening for the SimNet (which has $l_1/l_2$ similarities) was verified experimentally, whereas for the ConvNet, we observed that whitening does not have a positive effect (complying with the observations of~\cite{coates2011analysis}).

Fig.~\ref{fig:benchmark_results_templates} shows the cross-validation accuracies of the evaluated networks as a function of the number of templates ($n$).  Fig.~\ref{fig:benchmark_results_params} plots the same results against the number of learned parameters in the networks.  For the SimNet, we experimented with up to $400$ templates (we believe that more templates would only give marginal improvements in accuracy, and thus did not continue further), and reached accuracies of $76.8\%$ and $77.1\%$ with weighted $l_1$ and $l_2$ similarities respectively.  We ran the ConvNet with up to $6,400$ templates (beyond that Caffe had GPU memory management issues), with the highest accuracy standing at $76.2\%$.  In comparison, taking into account that with weighted similarities each template carries with it a weight vector, the size (number of learned parameters) of the $400$-template SimNet with weighted similarities was less than $1/9$ the size of the $6,400$-template ConvNet, while achieving slightly superior accuracy.  The performance of the SimNet with unweighted similarities on the other hand, is very similar to that of the ConvNet, thus highlighting the importance of the weights in the similarity layer.  The weights double the number of parameters in the layer, but the increase in performance scales up super-linearly with the number of added parameters.  In other words, weights provide a gain in accuracy which is much higher than what would be obtained by simply adding more templates until reaching the same network size.  For example, the accuracies with weights at $100$ templates are considerably higher than the accuracies without weights at $200$ templates, despite the fact that in the latter case, the overall network size is higher.

It is worth noting that the performance of the SimNet with unweighted $l_1$ and $l_2$ similarities is comparable to that of the ConvNet.  This confirms what we observed formally in subsec.~\ref{subsec:mlp} -- the hypothesis space (analyzed through the shapes of decision regions) corresponding to unweighted $l_2$-similarity is essentially the same as that which corresponds to linear similarity (convolutional operator).  The hypothesis space corresponding to unweighted $l_1$-similarity is different, but apparently does not provide a higher degree of abstraction (further study of this is deferred to future work).

Although the SimNet accuracies achieved here are not state of the art for this dataset, the results demonstrate the potential of SimNets for modeling learning problems with significant reduction in network sizes compared to ConvNets.
\subsection{Benchmarking against the ``single-layer'' network of Coates et al.}
The ``single-layer'' network studied by Coates et al.~(\cite{coates2011analysis}) is of interest on several accounts.  First, with GMM coding, the network is equivalent to the SimNet variant presented in eqn.~\ref{eqn:simnet_patch_class_2}.  Second, their best result with ``triangle'' coding is one of the highest accuracies on CIFAR-10 reported for networks of this depth (absolute state of the art in 2011).  Third, their observations with respect to the effect of whitening are relevant to the SimNet architecture, and indeed, we found that for the evaluated SimNet with $l_1$ and $l_2$ similarities, whitening makes a difference. 

In~\cite{coates2011analysis}, the network ``templates'' (i.e. the parameters of the selected coding method) were set using unlabeled data, and were not modified in the supervised training phase.  To facilitate a fair comparison against our SimNet (where templates are modified in supervised training), we added an additional supervised training phase, which applied to both the templates and the SVM coefficients.  More specifically, we used SGD to jointly modify the network templates and SVM coefficients produced by~\cite{coates2011analysis}, in an attempt to reach higher accuracy levels than those reported by the authors.  As it turned out, with the triangle coding they proposed, the supervised update of the templates did not improve accuracy any further than the original k-means clustering.  Deep inspection of this phenomena revealed that the k-means clustering (along with the SVM that follows) provides a strong local minima for the learning problem, so even the training accuracy was not improved.  This leads us to believe that the triangle coding is so successful precisely because it creates a representation for which k-means finds optimal templates, that cannot be improved even in the presence of labeled data.  In~\cite{coates2011analysis} results are reported for up to $1,600$ templates.  We used their code to reproduce these results, while running up to $6,400$ templates.  The accuracy curve we obtained is displayed in fig.~\ref{fig:benchmark_results_templates}, with $77.3\%$ for $1,600$ templates, $78.3\%$ for $3,200$ templates, and $77.8\%$ for $6,400$ templates.  The peak accuracy was achieved for $3,200$ templates, and was slightly higher than the SimNet peak accuracy, which stood at $77.1\%$ for weighted $l_2$-similarity and $400$ templates.  The SimNet on the other hand was almost $1/5$ in size (see fig.~\ref{fig:benchmark_results_params}).
\subsection{The importance of unsupervised initialization}
To assess the importance of the SimNets' unsupervised initialization scheme presented in sec.~\ref{sec:init}, we trained the evaluated SimNet (fig.~\ref{fig:exp_nets}(a)) with weighted $l_1$ similarity, using no data for initialization.  In particular, we initialized the templates $\z_1,...,\z_n$ randomly with a zero-mean unit-variance Gaussian distribution (in accordance with the fact that the input patches are whitened to have zero-mean and unit variance), and the weights $\uu_1,...,\uu_n$ with constant ones.  Besides the difference in initialization, the SimNet was trained exactly as described above.  Running the experiment with $200$ templates, cross-validation accuracy dropped from $76\%$ to $74.1\%$.  With $400$ templates, accuracy declined from $76.8\%$ to $74.4\%$.  With $50$ and $100$ templates, the learning algorithm did not converge.  We conclude that the SimNet unsupervised initialization scheme indeed has significant impact on performance.  The impact is especially acute for networks of small size.  This complies with conventional wisdom, according to which training small networks poses more difficult optimization problems.  The SimNet initialization scheme may provide an alternative to the common practice of over-specifying networks (constructing networks larger than necessary in order to ease the optimization task).

\section{Discussion}
We presented a deep layered architecture called SimNets, with similar ingredients as classical ConvNets.  The architecture is driven by two operators: \emph{(i)} the similarity operator, which is a generalization of the convolutional operator in ConvNets, and \emph{(ii)} the MEX operator, which can realize classical operators found in ConvNets like ReLU and max pooling, but has additional capabilities that make SimNets a powerful generalization of ConvNets.  One of the interesting properties of the SimNet architecture is that applying its two operators in succession -- similarity followed by MEX, results in what can be viewed as an artificial neuron in a high-dimensional feature space.  Moreover, the multilayer perceptron construction of input to hidden layer to output, as well as the fundamental building block incorporating locality, sharing and pooling, are both generalizations of kernel machines.  

We described two possible similarity measures: the $l_p$ similarity, which in its unweighted version gives rise to the Generalized Gaussian kernel, and the linear similarity, which is the operator found in ConvNets, and gives rise to the Exponential kernel.  We also showed that the full specification of the $l_p$ similarity operator, which includes weights, goes beyond a kernel machine and carries with it a higher abstraction level than what a convolutional layer can express.  Another interesting property of the SimNet architecture is that statistical estimation methods for Generalized Gaussian mixture distributions can be used for unsupervised initialization of network parameters.  These initializations arise naturally from standard statistical assumptions, having the potential of employing unsupervised learning in an effective manner as part of deep learning.

Implementing deep SimNets with state of the art optimization techniques (including GPU acceleration) is an ongoing effort, but we were able to implement a basic SimNet and conduct benchmarks comparing it against two networks of the same depth -- an analogous ConvNet and the ``single-layer'' network of~\cite{coates2011analysis}.  The results demonstrate that a SimNet can achieve comparable and/or better accuracy, while requiring a significantly smaller network (in terms of the number of learned parameters) -- around $1/9$ the size of the ConvNet and $1/5$ the size of the network in~\cite{coates2011analysis}.

The SimNet architecture departs from classical ConvNets in three main respects.  First, the similarity layer can incorporate entry-wise weights when the $l_p$ similarity is used.  With linear similarity (which is essentially an inner-product between an input patch and a convolutional kernel) incorporating weights is meaningless, as they blend into the convolutional kernels.  We saw that the unweighted $l_p$ and linear similarities give rise to a kernel-SVM building block with the Generalized Gaussian and Exponential kernels, respectively.  The weighted $l_p$ similarity on the other hand, cannot be realized in the kernel-SVM framework (thm.~\ref{thm:l_p_sim_wgt_no_K}), thereby offering a potentially stronger building block (whose effect is described in more detail in subsec.~\ref{subsec:mlp}).  The experiments we carried out highlight the differences between weighted and unweighted $l_p$ similarities:
\begin{itemize}
\item Without weights, $l_p$ similarity and the linear similarity (convolutional operator) give rise to comparable performance.  This suggests that without weights, SimNets do not exhibit superiority over ConvNets.
\item When weights are included, $l_p$ similarity displays a significant increase in performance, which scales up super-linearly with the number of parameters.  That is to say, the increase in accuracy cannot be explained merely by the fact that the number of parameters in the similarity layer has been doubled (weights on top of templates).
\end{itemize}
These findings suggest that the strength of having the basic building block go beyond the hypothesis space of a kernel-SVM, has significant appeal in practice. 

The second respect in which SimNets depart from ConvNets has to do with the ability of the MEX layer to incorporate offsets.  When the MEX layer serves as the final layer of the network, these offsets play the role of classification coefficients.  However, when the MEX layer is inserted as a pooling layer, the offsets can be interpreted as providing locality-based biases to the templates generated in a previous similarity layer.  This is something that classical ConvNets cannot express.  Evaluating the practical significance of the MEX offsets requires experimentation with deep layered SimNets, which is an ongoing effort.

The third departure (or distinction) from ConvNets, is that the SimNet architecture is endowed with a natural initialization based on unlabeled data.  In the case of ConvNets, existing unsupervised initialization schemes have little to no advantage over random initializations.  For the SimNets, we reported experimental results that demonstrate the superiority of the unsupervised initialization scheme over random initializations, showing that the effect is more acute when the networks are small.  Besides its aid in training, the unsupervised scheme proposed also has the potential of determining the number of channels for a similarity layer based on variance analysis of patterns generated from previous layers.  This implies that the structure of SimNets can potentially be determined automatically from (unlabeled) training data.   

Future work is focused on further implementation, with the purpose of creating an open programming environment for the research community, that will enable wider scale experimentation of SimNets.  Further theoretical studies are ongoing as well, with the intent to capture the sample complexity of SimNets, and to gain a better understanding of the typical network structure and size required under different conditions.

\subsubsection*{Acknowledgments}
The authors would like to thank Nitzan Guberman and Or Sharir for their dedicated contribution to the experiments carried out in this work.  The work is partly funded by Intel grant ICRI-CI no. 9-2012-6133 and by ISF Center grant 1790/12.

\subsubsection*{References}
{
\bibliographystyle{plainnat}
\bibliography{refs.bib}
}

\clearpage
\appendix
\section{Proof of theorem~\ref{thm:l_p_sim_wgt_no_K}} \label{app:proof_l_p_sim_wgt_no_K}
To prove the theorem, we will need the following definition and lemma: 
\\
\begin{definition} \label{def:eps_ortho_set}
Let $\HH$ be a Hilbert space and $S\subset\HH$ be a collection of vectors.  Given a constant $\epsilon>0$, $S$ is said to be an \emph{$\epsilon$-orthonormal set} if the following two conditions hold:
\beas
&\forall \vv\in S:~\norm{\vv}_\HH=1& \\
&\forall \vv,\vv'\in S~,~\vv\neq\vv':~\abs{\inprod{\vv}{\vv'}_\HH}\leq\epsilon&
\eeas
where $\inprod{\cdot}{\cdot}_\HH$ stands for the inner-product in $\HH$ and $\norm{\cdot}_\HH$ denotes the induced norm. \\
\end{definition}

\begin{lemma} \label{lemma:inf_ortho_set_prod_0}
Let $\HH$ be a Hilbert space over $F$ ($F=\R$ or $F=\C$) and $V\subset\HH$ be a set that contains an $\epsilon$-orthonormal subset (see def.~\ref{def:eps_ortho_set}) of size $n$ for any constants $\epsilon>0$ and $n\in\N$.  Then, for every vector $\uu\in\HH$ it holds that $\inf\left\{\abs{\inprod{\uu}{\vv}_\HH}:\vv\in V\right\}=0$.
\end{lemma}
\begin{proof}
Let $\uu\in\HH$ be an arbitrary vector, and denote $c:=\inf\left\{\abs{\inprod{\uu}{\vv}_\HH}:\vv\in V\right\}$ and $M:=\norm{\uu}_\HH$.  We would like to show that $c=0$.  Let $\epsilon>0$ and $n\in\N$ be arbitrary constants.  Given our assumption on $V$, we may choose an $\epsilon$-orthonormal subset $\vv_1,...,\vv_n\in V$.  Denote by $G$ the Gram of $\vv_1,...,\vv_n$, i.e. $G\in F^{n,n}$ is the positive semi-definite (PSD) matrix with entries $G_{ij}=\inprod{\vv_j}{\vv_i}_\HH$.  Let $\alpha_1,...,\alpha_n\in F$ be the scalars such that $\sum_{i=1}^n\alpha_i\vv_i$ is the projection of $\uu$ onto $span\{\vv_i\}_{i=1}^n$.  It then holds that:
\beas
&\norm{\sum_{i=1}^n\alpha_i\vv_i}_\HH^2\leq\norm{\uu}_\HH^2=M^2& \\
&\forall j\in\{1,...,n\}:\abs{\inprod{\sum_{i=1}^n\alpha_i\vv_i}{\vv_j}_\HH}=\abs{\inprod{\uu}{\vv_j}_\HH}\geq c& 
\eeas
If we denote $\alphabf:=[\alpha_1,...,\alpha_n]^\top$, and let $\1$ be the $n$-dimensional vector holding 1 in all entries, the above yields the following matrix inequalities:
\bea
&\alphabf^*G\alphabf\leq M^2& \label{eqn:G_form_up_bnd} \\
&\abs{G\alphabf}\geq c\cdot\1\implies\alphabf^*G^*G\alphabf=\norm{G\alphabf}_2^2\geq c^2\cdot n& \label{eqn:GG_form_low_bnd}
\eea
where $*$ stands for the conjugate transpose operator.  The matrix $G$ is PSD, thus having $n$ non-negative eigenvalues.  We denote these by $\lambda_1\geq...\geq\lambda_n\geq0$.  $G$ is Hermitian and thus $G^*G=G^2$, from which we readily conclude that $\lambda_1^2\geq...\geq\lambda_n^2\geq0$ are the eigenvalues of $G^*G$.  We thus have the following inequalities:
\bea
&\alphabf^*G\alphabf\geq\lambda_n\cdot\norm{\alphabf}_2^2& \label{eqn:G_form_low_bnd} \\
&\alphabf^*G^*G\alphabf\leq\lambda_1^2\cdot\norm{\alphabf}_2^2& \label{eqn:GG_form_up_bnd}
\eea
Combining the inequality~\ref{eqn:G_form_up_bnd} with~\ref{eqn:G_form_low_bnd}, and the inequality~\ref{eqn:GG_form_low_bnd} with~\ref{eqn:GG_form_up_bnd}, we get the following:
\bea
\lambda_n\cdot\norm{\alphabf}_2^2&\leq&M^2 \label{eqn:lambda_alpha_up_bnd} \\
\lambda_1^2\cdot\norm{\alphabf}_2^2&\geq&c^2\cdot n \label{eqn:lambda_alpha_low_bnd}
\eea
We now apply Gershgorin's circle theorem (see~\cite{golub2012matrix}) to $G$.  The theorem states that for each eigenvalue $\lambda_i$, there exists some $j\in\{1,...,n\}$ such that:
$$ \abs{\lambda_i-G_{jj}}\leq\sum_{j'\in\{1,...n\},j'\neq j}\abs{G_{jj'}} $$
Plugging in the definition of $G$ and the $\epsilon$-orthonormality of $\{\vv_1,...,\vv_n\}$, we get:
\beas
&\abs{\lambda_i-\underbrace{\inprod{\vv_j}{\vv_j}}_{=1}}\leq\sum\limits_{j'\in\{1,...n\},j'\neq j}\underbrace{\abs{\inprod{\vv_{j'}}{\vv_j}}}_{\leq\epsilon}\leq(n-1)\cdot\epsilon&  \\ 
&\implies1-(n-1)\cdot\epsilon\leq\lambda_i\leq1+(n-1)\cdot\epsilon& 
\eeas
Recall that $\epsilon>0$ and $n\in\N$ were chosen arbitrarily.  If we now limit $\epsilon$ to be smaller than $\frac{1}{n-1}$, we ensure that $\lambda_i>0$ for all $i=1,...,n$.  We can thus divide by $\lambda_n$ and $\lambda_1^2$ the inequalities~\ref{eqn:lambda_alpha_up_bnd} and~\ref{eqn:lambda_alpha_low_bnd} respectively, and reach: 
\beas
&\norm{\alphabf}_2^2\leq\frac{M^2}{\lambda_n}\leq\frac{M^2}{1-(n-1)\cdot\epsilon}& \\
&\norm{\alphabf}_2^2\geq\frac{c^2\cdot n}{\lambda_1^2}\geq\frac{c^2\cdot n}{(1+(n-1)\cdot\epsilon)^2}&
\eeas
Combining these two inequalities, we get:
$$ M^2\geq\frac{1-(n-1)\cdot\epsilon}{(1+(n-1)\cdot\epsilon)^2}\cdot c^2\cdot n$$
Now this holds for $\epsilon$ arbitrarily small, so in particular:
\be
M^2\geq\underbrace{\left(\lim_{\epsilon\to\0^+}\frac{1-(n-1)\cdot\epsilon}{(1+(n-1)\cdot\epsilon)^2}\right)}_{=1}\cdot c^2\cdot n=c^2\cdot n
\label{eqn:M_c_bnd}
\ee
$n$ is an arbitrary natural number, and $M$ was defined as the norm of $\uu\in\HH$ so in particular it is non-negative and finite.  In addition, $c$ was defined as $\inf\left\{\abs{\inprod{\uu}{\vv}_\HH}:\vv\in V\right\}$ so it is too non-negative.  Thus, the only way that eqn.~\ref{eqn:M_c_bnd} can hold is if $c=0$, which is what we set out to prove.
\end{proof}

Equipped with def.~\ref{def:eps_ortho_set} and lemma~\ref{lemma:inf_ortho_set_prod_0}, we head on to prove our main theorem:
\begin{proof}
[Proof of theorem~\ref{thm:l_p_sim_wgt_no_K}] Assume by contradiction that there are mappings $Z$ and $U$ and kernel $K$ as described in the theorem.  Let $\psi$ be a feature mapping corresponding to $K$, i.e. a mapping from $\R^d\times\R_+^d$ to some real Hilbert space $\HH$ such that $K([\z,\uu],[\z',\uu'])=\inprod{\psi([\z,\uu])}{\psi([\z',\uu'])}_\HH$ for all $\z,\z'\in\R^d$ and $\uu,\uu'\in\R_+^d$.  Fix some $\x\in\R^d$, and observe that:
\be
\norm{\psi([Z(\x),U(\x)])}_\HH\leq1
\label{eqn:x_map_small_norm}
\ee
This follows from:
\beas
&\norm{\psi([Z(\x),U(\x)])}_\HH^2=& \\
&K([Z(\x),U(\x)],[Z(\x),U(\x)])=& \\
&\exp\left\{-\underbrace{c}_{>0}\sum_{i=1}^d\underbrace{U(\x)_i}_{\geq0}\underbrace{\abs{x_i-Z(\x)_i}^p}_{\geq0}\right\}\leq1&
\eeas
Since $Z(\x)$ is also an element in $\R^d$, we can replace $\x$ by $Z(\x)$ in eqn.~\ref{eqn:x_map_small_norm} to obtain:
\be
\norm{\psi([Z(Z(\x)),U(Z(\x))])}_\HH\leq1
\label{eqn:Zx_map_small_norm}
\ee
Next we show that:
\be
\inprod{\psi([Z(Z(\x)),U(Z(\x))])}{\psi([Z(\x),U(\x)])}_\HH=1
\label{eqn:x_map_Zx_map_large_prod}
\ee
Indeed:
\beas
&\inprod{\psi([Z(Z(\x)),U(Z(\x))])}{\psi([Z(\x),U(\x)])}_\HH=& \\
&K([Z(Z(\x)),U(Z(\x))],[Z(\x),U(\x)])=& \\
&\exp\left\{-c\sum_{i=1}^d U(\x)_i\underbrace{\abs{Z(\x)_i-Z(\x)_i}^p}_{=0}\right\}=1&
\eeas
The Cauchy-Schwartz inequality (see~\cite{young1988introduction}) tells us that for any two vectors $\w,\w'$ in the real Hilbert space $\HH$, it holds that $\inprod{\w}{\w'}_\HH\leq\norm{\w}_\HH\cdot\norm{\w'}_\HH$.  Moreover, if equality holds then $\w$ and $\w'$ are linearly dependent, or more specifically, at least one of the vectors can be obtained by multiplying the other by a non-negative scalar.  Applying this to the vectors $\psi([Z(\x),U(\x)])$ and $\psi([Z(Z(\x)),U(Z(\x))])$, we conclude from equations~\ref{eqn:x_map_small_norm},~\ref{eqn:Zx_map_small_norm} and~\ref{eqn:x_map_Zx_map_large_prod} that:
\bea
\psi([Z(Z(\x)),U(Z(\x))])&=&\psi([Z(\x),U(\x)]) \label{eqn:x_map_Zx_map_equal} \\
\norm{\psi([Z(\x),U(\x)])}_\HH&=&1 \label{eqn:x_map_norm_1}
\eea
Using eqn.~\ref{eqn:x_map_Zx_map_equal} and our assumption about the kernel $K$ (eqn.~\ref{eqn:l_p_sim_wgt_K_req}), we conclude that for every $\z\in\R^d$ and $\uu\in\R_+^d$ (recall that $\x\in\R^d$ was fixed arbitrarily):
\beas
&\exp\left\{-c\sum_{i=1}^d u_i\abs{x_i-z_i}^p\right\}=& \\
&K\left([Z(\x),U(\x)],[\z,\uu]\right)=& \\
&\inprod{\psi([Z(\x),U(\x)])}{\psi([\z,\uu])}_\HH=& \\
&\inprod{\psi([Z(Z(\x)),U(Z(\x))])}{\psi([\z,\uu])}_\HH=& \\
&\exp\left\{-c\sum_{i=1}^d u_i\abs{Z(\x)_i-z_i}^p\right\}&
\eeas
Taking the logarithm of the two outer expressions, we get:
\beas
-c\sum_{i=1}^d u_i\abs{x_i-z_i}^p=-c\sum_{i=1}^d u_i\abs{Z(\x)_i-z_i}^p \\
\implies \sum_{i=1}^d u_i\left(\abs{x_i-z_i}^p-\abs{Z(\x)_i-z_i}^p\right)=0
\eeas
Fixing some coordinate $i_0\in\{1,...,d\}$, we can choose $\uu$ to hold $1$ at $i_0$ and $0$ in the other coordinates.  The latter equality would then reduce to $\abs{x_{i_0}-z_{i_0}}^p=\abs{Z(\x)_{i_0}-z_{i_0}}^p$, which must hold for any $z_{i_0}\in\R$.  The only way for this to be met is if $Z(\x)_{i_0}=x_{i_0}$.  Since both the vector $\x\in\R^d$ and the coordinate $i_0\in\{1,...,d\}$ are arbitrary, the mapping $Z:\R^d\to\R^d$ is no other than the identity mapping.  The assumption in eqn.~\ref{eqn:l_p_sim_wgt_K_req} thus becomes:
\bea
&\forall~\z,\x\in\R^d,\uu\in\R_+^d:& \label{eqn:l_p_sim_wgt_K_req_modf} \\
&K\left([\x,U(\x)],[\z,\uu]\right)=\exp\left\{-c\sum_{i=1}^d u_i\abs{x_i-z_i}^p\right\}& \nonumber
\eea
We again fix $\x\in\R^d$, and turn to show that $U(\x)\neq\0$ ($\0$ here stands for the $d$-dimensional zero vector).  Assume by contradiction that this is not the case, i.e. that $U(\x)=\0$.  Then, according to eqn.~\ref{eqn:l_p_sim_wgt_K_req_modf}, for all $\x'\in\R^d$ we have:
\beas
&\inprod{\psi([\x',U(\x')])}{\psi([\x,U(\x)])}_\HH=& \\
&K([\x',U(\x')],[\x,U(\x)])=& \\
&\exp\left\{-c\sum_{i=1}^d \underbrace{U(\x)_i}_{=0}\abs{x_i-x'_i}^p\right\}=1&
\eeas
Using the fact that $\psi([\x,U(\x)])$ and $\psi([\x',U(\x')])$ are unit vectors (eqn.~\ref{eqn:x_map_norm_1}), and again the Cauchy-Schwartz inequality, we conclude that $\psi([\x,U(\x)])=\psi([\x',U(\x')])$.  This implies that for all $\z\in\R^d$ and $\uu\in\R_+^d$:
\beas
&\exp\left\{-c\sum_{i=1}^d u_i\abs{x_i-z_i}^p\right\}=& \\
&K\left([\x,U(\x)],[\z,\uu]\right)=\inprod{\psi([\x,U(\x)])}{\psi([\z,\uu])}_\HH=& \\
&\inprod{\psi([\x',U(\x')])}{\psi([\z,\uu])}_\HH=K\left([\x',U(\x')],[\z,\uu]\right)=& \\
&\exp\left\{-c\sum_{i=1}^d u_i\abs{x'_i-z_i}^p\right\}&
\eeas
As before, we can isolate the coordinates in $\{1,...,d\}$ one at a time, and conclude that $\x=\x'$.  Since $\x'$ is arbitrary, this is of course a contradiction, showing that our assumption $U(\x)=\0$ was incorrect.  There is thus at least one coordinate of $U(\x)$ which is positive.  Accordingly, the expression $-c\sum_{i=1}^d U(\x)_i\abs{x'_i-x_i}^p$ will tend to $-\infty$ when all coordinates of $\x'$ tend to $\infty$ (we denote this condition by $\x'\to{\mathbf\infty}$).  We may thus write:
\beas
&\inprod{\psi([\x',U(\x')])}{\psi([\x,U(\x)])}_\HH=& \\
&\exp\left\{-c\sum_{i=1}^d U(\x)_i\abs{x'_i-x_i}^p\right\}\underset{\x'\to\infty}{\longrightarrow}0&
\eeas
Recall that $\x\in\R^d$ is an arbitrary vector.  The above convergence thus implies that for any $\epsilon>0$ and $n\in\N$, we can incrementally create a set of $n$ vectors - $\x_1,...,\x_n\in\R^d$, such that:
\beas
&\forall1\leq j<i\leq n:& \\
&\abs{\inprod{\psi([\x_i,U(\x_i)])}{\psi([\x_j,U(\x_j)])}_\HH}\leq\epsilon&
\eeas
Indeed, given a set of vectors $\x_1,...,\x_j$, the next vector $\x_{j+1}$ is obtained by approaching $\infty$ until all inner-products are small enough.

To summarize, we have the following findings:
\begin{itemize}
\item For any $\epsilon>0$ and $n\in\N$ there exist $\x_1,...,\x_n\in\R^d$ such that for all $1\leq j<i\leq n$, $\abs{\inprod{\psi([\x_i,U(\x_i)])}{\psi([\x_j,U(\x_j)])}_\HH}\leq\epsilon$.
\item $\norm{\psi([\x,U(\x)])}_\HH=1$ for all $\x\in\R^d$ (eqn.~\ref{eqn:x_map_norm_1}).
\item $\inprod{\psi([\x,U(\x)])}{\psi([\0,\0])}_\HH=K\left([\x,U(\x)],[\0,\0]\right)=\exp\left\{-c\sum_{i=1}^d 0\cdot\abs{x_i-0}^p\right\}=1$ (simply plug-in $\z=\0$ and $\uu=\0$ in eqn.~\ref{eqn:l_p_sim_wgt_K_req_modf}).
\end{itemize}
More succinctly, the set $V:=\{\psi([\x,U(\x)]):\x\in\R^d\}$ contains an $\epsilon$-orthonormal subset (def.~\ref{def:eps_ortho_set}) of size $n$ for any constants $\epsilon>0$ and $n\in\N$, and in addition the vector $\psi([\0,\0])$ has inner-product $1$ with every element of $V$.  According to lemma~\ref{lemma:inf_ortho_set_prod_0} this is impossible!  We have thus reached a contradiction, showing the incorrectness of our initial assumption that mappings $Z$ and $U$ and kernel $K$ as stated in the theorem exist.
\end{proof}

\section{Patch-based kernel-SVM} \label{app:patch_ksvm}
In this appendix we show how the classification described in eqn.~\ref{eqn:simnet_patch_class_3}, which corresponds to the basic ``locality-sharing-pooling'' SimNet illustrated in fig.~\ref{fig:layers_nets}(f), can be formulated as a multiclass kernel-SVM~(\cite{crammer2002algorithmic}) with reduced support-vectors~(\cite{wu2006direct}).  In this formulation, the classified instances will not be represented by holistic vectors, but rather by blocks of multiple vectors.  Moreover, the support-vectors will be subject to constraints which can be interpreted as enforcing ``locality'' and ``sharing''.  In the context of the SimNet, the vectors which constitute an instance are simply the input patches, the locality constraint on the support-vectors corresponds to the fact that the input is processed by local patches in a spatially aware manner, and the sharing constraint corresponds to the fact that the same $n$ templates in the similarity layer apply to all input patches.  As will be shown below, the SimNet's pooling operation will also come into play in the locality and sharing constraints.

Let $d\in\N$ be some dimension, and let $K:\R^d\times\R^d\to\R$ be a kernel on $\R^d$.  For some $D\in\N$, consider the instance space $\X:=\{X=(\x_1,...,\x_D):\x_i\in\R^d~,~i=1,...,D\}$.  For compatibility, we refer to the vectors that constitute an instance as ``patches''.  Assume we have a partitioning of patches into ``pools'', namely that there is a constant $P\in\N$ and a function $q:\{1,...,D\}\to\{1,...,P\}$ that assigns to each patch index $i\in\{1,...,D\}$ a pool $q(i)\in\{1,...,P\}$.  Consider the following rule for classifying an instance $X$ into one of $k\in\N$ possible classes:
\be
\hat{y}(X)=\argmax_{r=1,...,k}\sum_{1\leq p\leq P,1\leq l\leq n}\alpha_{rlp}\sum_{1\leq i\leq D:q(i)=p}K(\x_i,\z_l)
\label{eqn:patch_ksvm_rule}
\ee
where $n\in\N$ is some predetermined constant, $\{\z_1,...,\z_n\}\subset\R^d$ are learned templates, and $\{\alpha_{rlp}\}_{1\leq r\leq k,1\leq l\leq n,1\leq p\leq P}\subset\R$ are learned coefficients.  This is essentially equivalent to the SimNet classification described in eqn.~\ref{eqn:simnet_patch_class_3}.  In fact, the only true difference is that in the latter, the learned coefficients were constrained to be positive, but this does not limit generality, as we can always add a common offset to all coefficients after training is complete.

We now add the special character $*$ (``null'' character) to $\R^d$, extending the latter to $V:=\R^d\cup\{*\}$.  Accordingly, we extend $K$ to the function $K_V:V\times V\to\R$ defined by:
\be
K_V(\vv,\vv')=\left\{
	\begin{array}{ll}
		K(\vv,\vv')  & \mbox{if } \vv,\vv'\neq* \\
		0               & \mbox{otherwise}
	\end{array}
\right.
\label{eqn:K_V_def}
\ee
\begin{lemma} \label{lemma:K_V_kernel}
$K_V$ is a kernel on $V$.
\end{lemma}
\begin{proof}
Let $\psi$ be a feature mapping corresponding to the kernel $K$, i.e. $\psi$ is a mapping from $\R^d$ to some Hilbert space $\HH$ such that $\forall\x,\x'\in\R^d:K(\x,\x')=\inprod{\psi(\x)}{\psi(\x')}$.  Extend $\psi$ to the mapping $\psi_V:V\to\HH$ as follows:
$$ \psi_V(\vv)=\left\{
	\begin{array}{ll}
		\psi(\vv)  & \mbox{if } \vv\neq* \\
		0_\HH     & \mbox{if } \vv=*
	\end{array}
\right. $$
where $0_\HH$ stands for the zero element of $\HH$.  Obviously, the function from $V\times V$ to $\R$ defined by $(\vv,\vv')\mapsto\inprod{\psi_V(\vv)}{\psi_V(\vv')}$ is a kernel on $V$.  Direct computation shows that this function is no other than $K_V$:
\beas
&\inprod{\psi_V(\vv)}{\psi_V(\vv')}& \\
&=\left\{\begin{array}{ll}
		\inprod{\psi(\vv)}{\psi(\vv')}   & \mbox{if } \vv\neq*,\vv'\neq*   \\
		\inprod{\psi(\vv)}{0_\HH}       & \mbox{if } \vv\neq*,\vv'=*       \\
		\inprod{0_\HH}{\psi(\vv')}      & \mbox{if } \vv=*    ,\vv'\neq* \\
		\inprod{0_\HH}{0_\HH}          & \mbox{if } \vv=*    ,\vv'=*      
	   \end{array}\right.& \\
&=\left\{\begin{array}{ll}
		\inprod{\psi(\vv)}{\psi(\vv')}   & \mbox{if } \vv\neq*,\vv'\neq*   \\
		0                                       & \mbox{otherwise}
	   \end{array}\right.& \\
&=\left\{\begin{array}{ll}
		K(\vv,\vv')   & \mbox{if } \vv\neq*,\vv'\neq*   \\
		0                & \mbox{otherwise}
	   \end{array}\right.& \\
&=K_V(\vv,\vv')&
\eeas
\end{proof}
Next, we use the kernel $K_V$ to define a function $\K:V^D\times V^D\to\R$, where $V^D$ is the $D$'th Cartesian power of $V$, i.e. $V^D:=\{(\vv_1,...,\vv_D):\vv_i\in V~,~i=1,...,D\}$.  $\K$ is defined by:
\bea
&\K((\vv_1,...,\vv_D),(\vv'_1,...,\vv'_D))=\sum_{1\leq i\leq D} K_V(\vv_i,\vv'_i)& \nonumber \\
&=\sum_{1\leq i\leq D:\vv_i,\vv'_i\neq*} K(\vv_i,\vv'_i)& \label{eqn:Kbf_def}
\eea
\begin{lemma} \label{lemma:Kbf_kernel}
$\K$ is a kernel on $V^D$.
\end{lemma}
\begin{proof}
The proof is general in the sense that it does not rely on any specific property of the kernel $K_V$ on which $\K$ is based.  In accordance with the above notations, let $\psi_V$ be a feature mapping corresponding to $K_V$, i.e. a mapping from $V$ to some Hilbert space $\HH$ such that $\forall\vv,\vv'\in V:K_V(\vv,\vv')=\inprod{\psi_V(\vv)}{\psi_V(\vv')}$.  We use $\psi_V$ to define a mapping $\Psi$ from $V^D$ to the Hilbert space $\HH^D$, which is the $D$'th Cartesian power of $\HH$ (the elements of $\HH^D$ are $D$-length sequences of $\HH$-elements, and the inner product between $(\h_1,...,\h_D)$ and $(\h'_1,...,\h'_D)$ is defined as $\sum_{i=1}^D\inprod{\h_i}{\h'_i}$).  The mapping $\Psi$ is defined by $\Psi((\vv_1,...,\vv_D))=(\psi_V(\vv_1),...,\psi_V(\vv_D))$.  The function from $V^D\times V^D$ to $\R$ defined by $((\vv_1,...,\vv_D),(\vv'_1,...,\vv'_D))\mapsto\inprod{\Psi((\vv_1,...,\vv_D))}{\Psi((\vv'_1,...,\vv'_D))}$ is obviously a kernel on $V^D$.  Direct computation shows that this function is no other than $\K$:
\beas
&\inprod{\Psi((\vv_1,...,\vv_D))}{\Psi((\vv'_1,...,\vv'_D))}&  \\
&=\inprod{(\psi_V(\vv_1),...,\psi_V(\vv_D))}{(\psi_V(\vv'_1),...,\psi_V(\vv'_D))}& \\
&=\sum_{i=1}^D\inprod{\psi_V(\vv_i)}{\psi_V(\vv'_i)}=\sum_{i=1}^D K_V(\vv_i,\vv'_i)& \\
&=\K((\vv_1,...,\vv_D),(\vv'_1,...,\vv'_D))&
\eeas
\end{proof}
Using $\K$, we may express the classification rule given in eqn.~\ref{eqn:patch_ksvm_rule} in kernel-form.  Simply notice that:
\beas
&\sum_{1\leq i\leq D:q(i)=p}K(\x_i,\z_l)=& \\
&\K((\x_1,...,\x_D),\overbrace{(\underbrace{*,...,*}_{q(i)\neq p},\underbrace{\z_l,...,\z_l}_{q(i)=p},\underbrace{*,...,*}_{q(i)\neq p})}^{:=Z_{lp}})&
\eeas
Thus, eqn.~\ref{eqn:patch_ksvm_rule} can be written as follows:
$$ \hat{y}(X)=\argmax_{r=1,...,k}\sum_{1\leq p\leq P,1\leq l\leq n}\alpha_{rlp}\cdot\K(X,Z_{lp}) $$
where $Z_{lp}=\left((Z_{lp})_1,...,(Z_{lp})_D\right)$, for $l=1,...,n$ and $p=1,...,P$, are elements in $V^D$ meeting the following constraints:
\be
\forall i\in\{1,...,D\}~s.t.~q(i)\neq p:~(Z_{lp})_i=*
\label{eqn:patch_ksvm_local_constr}
\ee
\vspace{-7mm}
\bea
&\forall i\in\{1,...,D\}~s.t.~q(i)=p:~(Z_{lp})_i=\z_l& \label{eqn:patch_ksvm_share_constr} \\
&\text{for some global vectors } \z_1,...,\z_n\in\R^d& \nonumber
\eea
We interpret the constraint in eqn.~\ref{eqn:patch_ksvm_local_constr} as enforcing locality -- entries of $Z_{lp}$ that lie outside the pool $p$ (``out-pool'' entries) must hold the null character.  The constraint in eqn.~\ref{eqn:patch_ksvm_share_constr} is interpreted as enforcing sharing -- entries of $Z_{lp}$ that lie inside the pool $p$ (``in-pool'' entries) are identical to each other, and also to the in-pool entries of $Z_{l'p'}$ in the case where the ``template indexes'' $l$ and $l'$ are the same. 

To conclude, the classifier described in eqn.~\ref{eqn:patch_ksvm_rule} can also be expressed as:
\be
\hat{y}(X)=\argmax_{r=1,...,k}\sum_{1\leq p\leq P,1\leq l\leq n}\alpha_{rlp}\cdot\K(X,Z_{lp})
\label{eqn:patch_ksvm_rule_modf}
\ee
where:
\begin{itemize}
\item The set $V^D$ is simply the instance space $\X$ with the option of placing null characters in the different entries.
\item $\K$ is a kernel on $V^D$.
\item $Z_{lp}$, with $l=1,...,n$ and $p=1,...P$, are learned elements of $V^D$ meeting the locality and sharing constraints in eqn.~\ref{eqn:patch_ksvm_local_constr} and eqn.~\ref{eqn:patch_ksvm_share_constr} respectively.
\item $\alpha_{rlp}$, with $r=1,...,k$, $l=1,...,n$ and $p=1,...P$, are learned real coefficients.
\end{itemize}
That is to say, the classifier is a reduced kernel-SVM on the space $V^D$ with the kernel $\K$, where the train and test instances are known to lie in the subset $\X\subset V^D$ (i.e. they do not contain any null characters), and there are $n\cdot P$ support-vectors indexed by $(l,p)\in\{1,...,n\}\times\{1,...,P\}$, that are subject to the locality and sharing constraints in eqn.~\ref{eqn:patch_ksvm_local_constr} and eqn.~\ref{eqn:patch_ksvm_share_constr} respectively.  This construction, which we refer to as ``patch-based kernel-SVM'', underlines the strong connection between SimNets and kernel machines.  In particular, it demonstrates the effect of locality, sharing and pooling in SimNets on the kernel-SVM equivalent.  Namely, while the basic SimNet (illustrated in fig.~\ref{fig:layers_nets}(e)) was associated with standard reduced kernel-SVM, adding locality, sharing and pooling to obtain the SimNet considered here (illustrated in fig.~\ref{fig:layers_nets}(f)), translates the associated kernel machine to patch-based kernel-SVM, in which the concepts of locality, sharing and pooling come into play as constraints on the support-vectors.

\end{document}